\documentclass{article}

\usepackage{microtype}
\usepackage{graphicx}
\usepackage{subfigure}
\usepackage{booktabs} 
\usepackage{amsmath,amssymb,amsfonts,amsthm}
\usepackage{multirow}
\usepackage[utf8]{inputenc}

\usepackage{hyperref}

\theoremstyle{plain}

\newtheorem{proposition}{Proposition}[section]
\theoremstyle{definition}

\newtheorem{definition}{Definition}[section]

\theoremstyle{remark}
\newtheorem{remark}{Remark}[section]

\DeclareMathOperator*{\argmin}{arg\,min}

\usepackage[accepted]{icml2021}

\icmltitlerunning{Explaining Time Series Predictions with Dynamic Masks}

\begin{document}

\twocolumn[
\icmltitle{Explaining Time Series Predictions with Dynamic Masks}



\icmlsetsymbol{equal}{*}

\begin{icmlauthorlist}
\icmlauthor{Jonathan Crabbé}{cam}
\icmlauthor{Mihaela van der Schaar}{cam,ucla,ati}

\end{icmlauthorlist}

\icmlaffiliation{cam}{DAMTP, University of Cambridge, UK}
\icmlaffiliation{ucla}{University of California Los Angeles, USA}
\icmlaffiliation{ati}{The Alan Turing Institute, UK}

\icmlcorrespondingauthor{Jonathan Crabbé}{jc2133@cam.ac.uk}
\icmlcorrespondingauthor{Mihaela van der Schaar}{mv472@cam.ac.uk}

\icmlkeywords{Machine Learning, ICML}

\vskip 0.3in
]



 \printAffiliationsAndNotice{} 
 
\newcommand{\N}{\mathbb{N}}
\newcommand{\R}{\mathbb{R}}
\newcommand{\partderiv}[2]{\frac{\partial #1}{\partial #2}}

\begin{abstract}
How can we explain the predictions of a machine learning model? When the data is structured as a multivariate time series, this question induces additional difficulties such as the necessity for the explanation to embody the time dependency and the large number of inputs. To address these challenges, we propose dynamic masks (Dynamask). This method produces instance-wise importance scores for each feature at each time step by fitting a perturbation mask to the input sequence. In order to incorporate the time dependency of the data, Dynamask studies the effects of dynamic perturbation operators. In order to tackle the large number of inputs, we propose a scheme to make the feature selection parsimonious~(to select no more feature than necessary) and legible~(a notion that we detail by making a parallel with information theory). With synthetic and real-world data, we demonstrate that the dynamic underpinning of Dynamask, together with its parsimony, offer a neat improvement in the identification of feature importance over time. The modularity of Dynamask makes it ideal as a plug-in to increase the transparency of a wide range of machine learning models in areas such as medicine and finance, where time series are abundant.

\end{abstract}

\section{Introduction and context}
\label{sec-intro}
What do we need to trust a machine learning model? If accuracy is necessary, it might not always be sufficient. With the application of machine learning to critical areas such as medicine, finance and the criminal justice system, the black-box nature of modern machine learning models has appeared as a major hindrance to their large scale deployment~\citep{Caruana2015, Lipton2016, Ching2018}. With the necessity to address this problem, the field of explainable artificial intelligence (XAI) thrived~\citep{BarredoArrieta2020, Das2020, Tjoa2020}.

\textbf{Saliency Methods} Among the many possibilities to increase the transparency of a machine learning model, we focus here on \emph{saliency methods}. The purpose of these methods is to highlight the features in an input that are relevant for a model to issue a prediction. We can distinguish them according to the way they interact with a model to produce importance scores.

\emph{Gradient-based:} These methods use the gradient of the model's prediction with respect to the features to produce importance scores. The premise is the following: if a feature is salient, we expect it to have a big impact on the model's prediction when varied locally. This translates into a big gradient of the model's output with respect to this feature. Popular gradient methods include Integrated Gradient~\citep{Sundararajan2017}, DeepLIFT~\citep{Shrikumar2017} and GradSHAP~\citep{Lundberg2018}. 

\emph{Perturbation-based:} These methods use the effect of a perturbation in the input on the model's prediction to produce importance scores. The premise is similar to the one used in gradient based method. The key difference lies in the way in which the features are varied. If gradient based methods use local variation of the features (according to the gradient), perturbation-based method use the data itself to produce a variation. A first example is Feature Occlusion~\citep{Suresh2017} that replaces a group of features with a baseline. Another example is Feature Permutation that performs individual permutation of features within a batch.

\emph{Attention-based:} For some models, the architecture allows to perform simple explainability tasks, such as determining feature saliency. A popular example, building on the success attention mechanisms~\citep{Vaswani2017}, is the usage of attention layers to produce importance scores~\citep{Choi2016, Song2017, Xu2018, Kwon2018}. 

\emph{Other:} There are some methods that don't clearly fall in one of the above categories. A first example is SHAP~\citep{Lundberg2017}, which attributes importance scores based on Shapley values. Another popular example is LIME in which the importance scores correspond to weights in a local linear model~\citep{Ribeiro2016}. Finally, some methods such as INVASE~\citep{Yoon2019} or ASAC~\citep{Yoon2019b} train a selector network to highlight important features.

\textbf{Time Series Saliency} Saliency methods were originally introduced in the context of image classification~\citep{Simonyan2013}. Since then, most methods have focused on images and tabular data. Very little attention has been given to time series~\citep{BarredoArrieta2020}. A possible explanation for this is the increasing interest in \emph{model agnostic methods}~\citep{Ribeiro2016b}.

Model agnostic methods are designed to be used with a very wide range of models and data structures. In particular, nothing prevents us from using these methods for \emph{Recurrent Neural Networks}~(RNN) trained to handle multivariate time series. For instance, one could compute the Shapley values induced by this RNN for each input $x_{t,i}$ describing a feature $i$ at time $t$. In this configuration, all of these inputs are considered as individual features and the time ordering is forgotten. This approach creates a conceptual problem illustrated in Figure~\ref{illustration-context}.

\begin{figure}[ht]
\vskip 0.2in
\begin{center}
\centerline{\includegraphics[width=0.9\columnwidth]{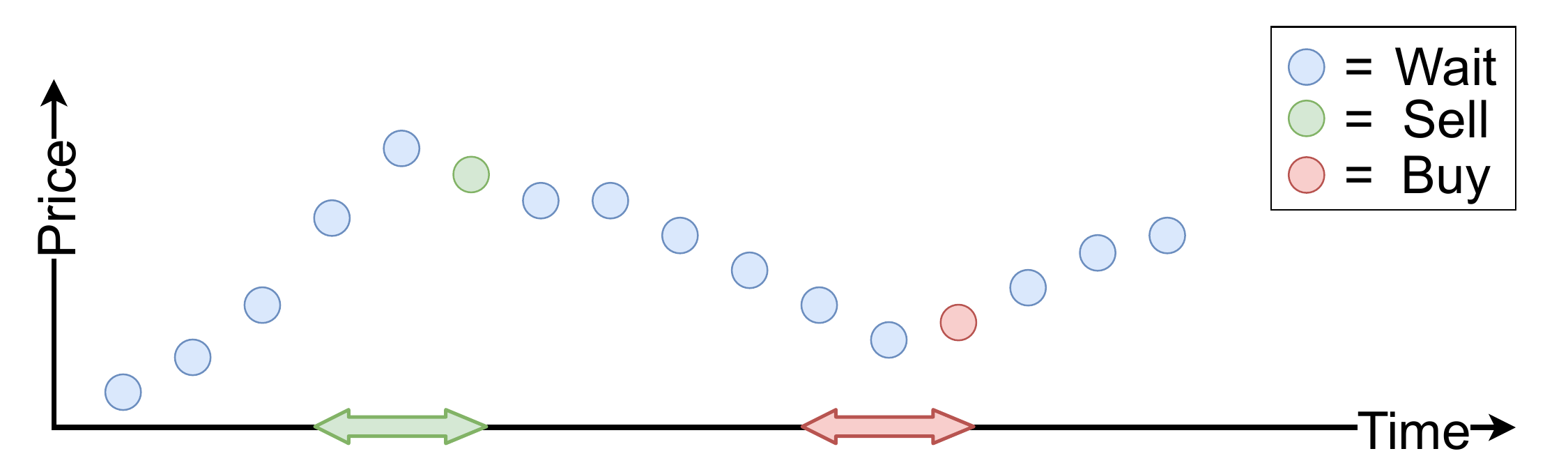}}
\caption{Context matters. This graph shows a fictional security price over time. A simplistic model recommends to sell just after a local maximum, to buy just after a local minimum and to wait otherwise. When recommending to sell or to buy, 3 consecutive time steps are required as a context.}
\label{illustration-context}
\end{center}
\vskip -0.2in
\rule[1ex]{\columnwidth}{0.5pt}
\end{figure}   

In this example, the time dependency is crucial for the model. At each time step, the model provides a decision based on the two previous time steps. This oversimplified example illustrates that the time dependency induces a context, which might be crucial to understand the prediction of some models. In fact, recurrent models are inherently endowed with this notion of context as they rely on memory cells. This might explain the poor performances of model agnostic methods to identify salient features for RNNs~\citep{Ismail2020}. These methods are \emph{static} as the time dependency, and hence the context, is forgotten when treating all the time steps as separate features. 

If we are ready to relax the model agnosticism, we can take a look at attention based methods. In methods such as RETAIN~\citep{Choi2016}, attention weights are interpreted as importance scores. In particular, a high attention weight for a given input indicates that this input is salient for the model to issue a prediction. In this case, the architecture of the model allows us to provide importance scores without altering the dynamic nature of the data. However, it has been shown recently that the attention weights of a model can significantly be changed without inducing any effect on the predictions~\citep{Jain2019}. This renders the parallel between attention weights and feature importance rather unclear. This discrepancy also appears in our experiments.  

Another challenge induced by time series is the large number of inputs. The number of features is multiplied by the number of time steps that the model uses to issue a prediction. A saliency map can therefore become quickly overwhelming in this setting. To address these challenges, it is necessary to incorporate some treatment of \emph{parsimony} and \emph{legibility} in a time series saliency method. By \emph{parsimony}, we mean that a saliency method should not select more inputs than necessary to explain a given prediction. For instance, a method that identifies $90 \%$ of the inputs as salient is not making the interpretation task easier. By \emph{legibility}, we mean that the analysis of the importance scores by the user should be as simple as possible. Clearly, for long time sequences, analysing feature maps such as (c) and (d) on Figure~\ref{illustration-mask_information} can quickly become daunting.  

To address all of these challenges, we introduce Dynamask. It is a \emph{perturbation-based} saliency method building on the concept of \emph{masks} to produce post-hoc explanations for any time series model. Masks have been introduced in the context of image classification~\citep{Fong2017, Fong2019}. In this framework, a mask is fitted to each image. The mask highlights the regions of the image that are salient for the black-box classifier to issue its prediction. These masks are obtained by perturbing the pixels of the original image according to the surrounding pixels and study the impact of such perturbations on the black-box prediction. It has been suggested to extend the usage of masks beyond image classification~\citep{Phillips2018, Ho2019}. However, to our knowledge, no available implementation and quantitative comparison with benchmarks exists in the context of multivariate time series. Moreover, few works in the literature explicitly mention explainability in a multivariate time series setting~\citep{Siddiqui2019, Tonekaboni2020}. Both of these considerations motivate our proposition.

\textbf{Contributions} By building on the challenges that have been described, our work is, to our knowledge, the first saliency method to rigorously address the following questions in a time series setting.\\
\emph{(1) How to incorporate the context?} In our framework, this is naturally achieved by studying the effect of \emph{dynamic} perturbations. Concretely, a perturbation is built for each feature at each time by using the value of this feature at adjacent times. This allows us to build meaningful perturbations that carry contexts such as the one illustrated in Figure~\ref{illustration-context}.\\   
\emph{(2) How to be parsimonious?} A great advantage of using masks is that the notion of parsimony naturally translates into the \emph{extremal} property. Conceptually, an \emph{extremal mask} selects the minimal number of inputs allowing to reconstruct the black-box prediction with a given precision. \\ 
\emph{(3) How to be legible?} To make our masks as simple as possible, we encourage them to be almost binary. Concretely, this means that we enforce a polarization between low and high saliency scores. Moreover, to make the notion of legibility quantitative, we propose a parallel with information theory\footnote{Many connections exist between explainability and information theory, see for instance~\citep{Chen2018}.}. This allows us to introduce two metrics : the mask information and the mask entropy. As illustrated in Figure~\ref{illustration-mask_information}, the entropy can be used to assess the legibility of a given mask. Moreover, these metrics can also be computed for other saliency methods, hence allowing insightful comparisons.

The paper is structured as follows. In Section~\ref{sec-math}, we outline the mathematical formalism related to our method. Then, we evaluate our method by comparing it with several benchmarks in Section~\ref{sec-experiments} and conclude in Section~\ref{sec-conclusion}.

\section{Mathematical formulation}
\label{sec-math}
\begin{figure*}[t]
\vskip 0.2in
\begin{center}
\centerline{\includegraphics[width=0.7\textwidth]{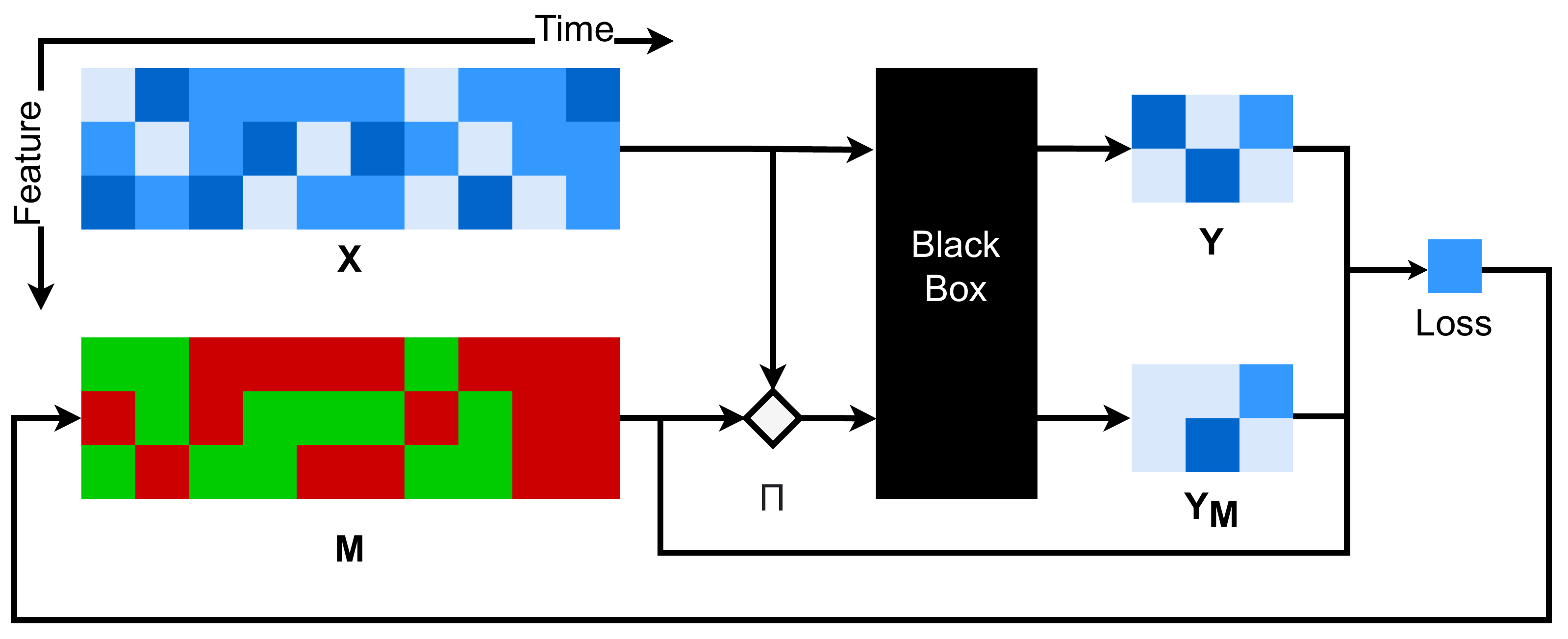}}
\caption{Diagram for Dynamask. An input matrix $\textbf{X}$, extracted from a multivariate time series, is fed to a black-box to produce a prediction $\textbf{Y}$. The objective is to give a saliency score for each component of $\textbf{X}$. In Dynamask, these saliency scores are stored in a mask $\textbf{M}$ of the same shape as the input $\textbf{X}$. To detect the salient information in the input $\textbf{X}$, the mask produces a perturbed version of $\textbf{X}$ via a perturbation operator $\Pi$. This perturbed $\textbf{X}$ is fed to the black-box to produce a perturbed prediction $\textbf{Y}_{\textbf{M}}$. The perturbed prediction is compared to the original prediction and the error is backpropagated to adapt the saliency scores contained in the mask.}
\label{illustration-dynamask}
\end{center}
\vskip -0.2in
\rule[1ex]{\textwidth}{0.5pt}
\end{figure*}

In this section, we formalize the problem of feature importance over time as well as the proposed solution. For the following, it is helpful to have the big picture in mind. Hence, we present the blueprint of Dynamask in Figure~\ref{illustration-dynamask}. The rest of this section makes this construction rigorous.

\subsection{Preliminaries}

Let $\mathcal{X} \subset \mathbb{R}^{d_X}$ be an input (or feature) space and $\mathcal{Y} \subset \mathbb{R}^{d_Y}$ be an output (or label) space, where $d_X$ and $d_Y$ are respectively the dimension of the input and the output space. For the sake of concision, we denote by $[n_1:n_2]$ the set of natural numbers between the natural numbers $n_1$ and $n_2$ with $ n_1 < n_2$. We assume that the data is given in terms of time series $(\textbf{x}_t)_{t \in [1:T]} $, where the inputs $\textbf{x}_t \in \mathcal{X}$ are indexed by a time parameter $t \in [1:T]$ with $T \in \mathbb{N^*}$. We consider the problem of predicting a sequence $(\textbf{y}_t)_{t \in [t_y:T]}$ of outputs $\textbf{y}_t \in \mathcal{Y}$ indexed by the same time parameter, but starting at a time $t_y \geq 1$, hence allowing to cover sequence to vector predictions ($t_y = T$) as well as sequence to sequence predictions ($ t_y < T$). For the following, it is convenient to introduce a matrix notation for these time series. In this way, $\mathbf{X}=(x_{t,i})_{(t,i) \in [1:T]\times [1:d_X]}$ denotes the matrix in $\mathbb{R}^{T \times d_X}$ whose rows correspond to time steps and whose columns correspond to features. Similarly\footnote{In the following, we do not make a distinction between $ \mathbb{R}^{T \times d_X} $ and $ \mathcal{X}^T $ as these are isomorphic vector spaces. The same goes for $ \mathbb{R}^{(T + 1 - t_y) \times d_Y} $ and $ \mathcal{Y}^{T + 1 - t_y} $.} , we denote $\mathbf{Y}=(y_{t,i})_{(t,i) \in [t_y:T]\times [1:d_Y]}$ for the matrix in $\mathbb{R}^{(T+1-t_y) \times d_Y}$.

Our task is to explain the prediction $\textbf{Y} = f(\textbf{X})$ of a black box $f$ that has been pre-trained for the aforementioned prediction task. In other words, our method aims to explain individual predictions of a given black box. More specifically, our purpose is to identify the parts of the input $\textbf{X}$ that are the most relevant for $f$ to produce the prediction $\textbf{Y}$. To do this, we use the concept of masks introduced in \citep{Fong2017, Fong2019}, but in a different context. Here, we adapt the notion of mask to a dynamic setting.

\begin{definition}[Mask]
A mask associated to an input sequence $\textbf{X} \in \mathbb{R}^{T \times d_X}$ and a black box $f :\mathcal{X}^T \rightarrow \mathcal{Y}^{T+1-t_y} $ is a matrix $\textbf{M} = (m_{t,i}) \in [0,1]^{T \times d_X}$ of the same dimension as the input sequence. The element $m_{t,i}$ of this matrix represents the importance of feature $i$ at time $t$ for $f$ to produce the prediction $\textbf{Y} = f(\textbf{X})$.    
\end{definition}
\begin{remark}
For a given coefficient in the mask matrix, a value close to $1$ indicates that the feature is salient, while a value close to $0$ indicates the opposite.
\end{remark}

Now how can we obtain such a mask given a black box and an input sequence? To answer this question, we should let the mask act on the inputs and measure the effect it has on a black box prediction. As the name suggests, we would like this mask to hide irrelevant inputs contained in $\textbf{X}$. To make this rigorous, it is useful to spend some time to think about perturbation operators.

\subsection{Perturbation operators}

The method that we are proposing in this paper is perturbation-based. Concretely, this means that a mask is used to build a \emph{perturbation operator}. Examples of such operators are given in \citep{Fong2019} in the context of image classification. Here, we propose a general definition and we explain how to take advantage of the dynamic nature of the data to build meaningful perturbations. By recalling that the mask coefficients indicate the saliency, we expect the perturbation to vanish for features $x_{t,i}$ whose mask coefficient $m_{t,i}$ is close to one. This motivates the following definition.

\begin{definition}[Perturbation operator] \label{def-perturbation_operator}
A perturbation operator associated to a mask $\textbf{M} \in [0,1]^{T \times d_X}$ is a linear operator acting on the input sequence space $\Pi_{\textbf{M}}~:~\mathbb{R}^{T \times d_X}~\rightarrow~\mathbb{R}^{T \times d_X}$. It needs to fulfil the two following assumptions for any given $(t,i) \in [1:T] \times [1:d_X] $:

1. The perturbation for $x_{t,i}$ is dictated by $m_{t,i}$ :
\begin{align*}
\left[ \Pi_{\textbf{M}}\left( \textbf{X} \right) \right]_{t,i} = \pi \left(\textbf{X} , m_{t,i} \ ; t, i \right),
\end{align*}
where $\pi$ is differentiable for $m \in (0,1)$ and continuous at $m  =0,1$.

2. The action of the perturbation operator is trivial when the mask coefficient is set to one :
\begin{align*}
\pi \left(\textbf{X} ,1 \ ; t, i \right) = x_{t,i}.
\end{align*}

Moreover, we say that the perturbation is \emph{dynamic} if for any given feature at any given time, the perturbation is constructed with the values that this feature takes in neighbouring times. More formally, the perturbation is dynamic if there exist a couple $ (W_1 ,W_2)  \in \N \times \N \setminus \{ (0,0)\}$ such that for all $(t,i) \in [1:T] \times [1:d_X]$ we have:
\begin{align*}
\partderiv{\left[ \pi \left(\textbf{X} , m_{t,i} \ ; t, i \right) \right]}{x_{t',i}} \neq 0  
\end{align*}
for all $ t' \in [t-W_1 : t+W_2] \cap [1:T]$.
\end{definition}
\begin{remark}
The first assumption ensures that the perturbations are applied independently for all inputs and that our method is suitable for gradient-based optimization. The second assumption\footnote{Together with the continuity of $\pi$ with respect to the mask coefficients.} translates the fact that the perturbation should have less effect on salient inputs. 
\end{remark}
\begin{remark}
The parameters $W_1$ and $W_2$ control how the perturbation depends on neighbouring times. A non-zero value for $W_1$ indicates that the perturbation depends on past time steps. A non-zero value for $W_2$ indicates that the perturbation depends on future time steps. For the perturbation to be dynamic, at least one of these parameters has to be different from zero. 
\end{remark}

This definition gives the freedom to design perturbation operators adapted to particular contexts. The method that we develop can be used for any such perturbation operator. In our case, we would like to build our masks by taking advantage the dynamic nature of the data into account. This is crucial if we want to capture local time variations of the features, such as a quick increase of the blood-pressure or an extremal security price from Figure~\ref{illustration-context}. This is achieved by using dynamic perturbation operators. To illustrate, we provide three examples:
\begin{align*}
& \pi^g \left(\textbf{X} , m_{t,i} \ ; t,i \right) = \frac{\sum_{t' = 1}^T x_{t' , i} \cdot g_{\sigma(m_{t,i})}(t-t')}{\sum_{t' = 1}^T g_{\sigma(m_{t,i})}(t-t')}   \\
& \pi^m \left(\textbf{X} , m_{t,i} \ ; t,i \right) = m_{t,i} \cdot x_{t,i} + \left( 1 - m_{t,i} \right)\cdot \mu_{t,i} \\
& \pi^p \left(\textbf{X} , m_{t,i} \ ; t,i \right) = m_{t,i} \cdot x_{t,i} + \left( 1 - m_{t,i} \right)\cdot \mu^p_{t,i} ,
\end{align*}
where $\pi^g$ is a temporal Gaussian blur~\footnote{For this perturbation to be continuous, we assume that $\pi^g(\textbf{X}, 1 \ ; t, i ) \equiv x_{t,i}$.} with
\begin{align*}
g_{\sigma}(t) = \exp \left( - \frac{t^2}{2 \sigma^2}  \right) ; \ \sigma(m) = \sigma_{max} \cdot(1-m).
\end{align*}
 Similarly, $\pi^m$ can be interpreted as a fade to moving average perturbation with
\begin{align*}
\mu_{t,i} = \frac{1}{2W + 1} \sum_{t'=t-W}^{t+W} x_{t',i},
\end{align*}
where $W \in \mathbb{N}$ is the size of the moving window. Finally, $\pi^p$ is similar to $\pi^m$ with one difference: the former only uses past values of the features to compute the perturbation\footnote{This is useful in a typical forecasting setting where the future values of the feature are unknown.}: 
\begin{align*}
\mu^p_{t,i} = \frac{1}{W + 1} \sum_{t'=t-W}^{t} x_{t',i}.
\end{align*}
Note that the Hadamard product $\textbf{M} \odot \textbf{X}$ used in \citep{Ho2019} is a particular case of $\pi^m$ with $\mu_{t,i} = 0=$ so that this perturbation is static. In the following, to stress that a mask is obtained by inspecting the effect of a dynamic perturbation operator, we shall refer to it as a dynamic mask (or Dynamask in short). 

In terms of complexity, the computation of $\pi^g $ requires $\mathcal{O}(d_X \cdot T^2)$ operations~\footnote{One must compute a perturbation for each  $d_X \cdot T$ component of the input and the sums in each perturbation have $T$ terms} while $\pi^m$ requires $\mathcal{O}(d_X \cdot T \cdot W)$ operations~\footnote{One must compute a perturbation for each  $d_X \cdot T$ component of the input and each moving average involves $2W+1$ terms}. When $T$ is big\footnote{In our experiments, however, we never deal with $T > 100$ so that both approaches are reasonable.}, it might therefore be more interesting to use $\pi^m$ with $W \ll T$ or a windowed version of $\pi^g$. With this analysis of perturbation operators, everything is ready to explain how dynamic masks are obtained.   

\subsection{Mask optimization}
\label{subsec-optimization}

To design an objective function for our mask, it is helpful to keep in mind what an ideal mask does. From the previous subsection, it is clear that we should compare the black-box predictions for both the unperturbed and the perturbed input. Ideally, the mask will identify a subset of salient features contained in the input that explains the black-box prediction. Since this subset of features is salient, the mask should indicate to the perturbation operator to preserve it. More concretely, a first part of our objective function should keep the shift in the black-box prediction to be small. We call it the \emph{error} part of our objective function. In practice, the expression for the error part depends on the task done by the black-box. In a regression context, we minimize the squared error between the unperturbed and the perturbed prediction:
\begin{align*} 
\mathcal{L}_e \left( \textbf{M} \right) =  \sum_{t=t_y}^T \sum_{i=1}^{d_Y} \left( \left[  \left( f \circ \Pi_{\textbf{M}} \right)(\textbf{X}) \right]_{t,i} - \left[f(\textbf{X})\right]_{t,i} \right)^2 .
\end{align*} 
Similarly, in a classification task, we minimize the cross-entropy between the predictions:
\begin{align*} 
\mathcal{L}_e \left( \textbf{M} \right) = - \sum_{t=t_y}^T \sum_{c=1}^{d_Y} \left[ f(\textbf{X}) \right]_{t,c} \log \left[ \left( f \circ \Pi_{\textbf{M}} \right)(\textbf{X}) \right]_{t,c}.
\end{align*}
Now we have to make sure that the mask actually selects salient features and discards the others. By remembering that $m_{t,i}=0$ indicates that the feature $x_{t,i}$ is irrelevant for the black-box prediction, this selection translates into imposing sparsity in the mask matrix $\textbf{M}$. A first approach for this, used in \citep{Fong2017, Ho2019}, is to add a $l^1$ regularisation term on the coefficients of $\textbf{M}$. However, it was noted in \citep{Fong2019} that this produces mask that vary with the regularization coefficient $\lambda$  in a way that renders comparisons between different $\lambda$ difficult. To solve this issue, they introduce a new regularization term to impose sparsity:
\begin{align*}
\mathcal{L}_a \left( \textbf{M} \right) = \Vert \text{vecsort}(\textbf{M}) - \textbf{r}_{a} \Vert^2,
\end{align*}  
where $\Vert.\Vert$ denotes the vector 2-norm, $\text{vecsort}$ is a function that vectorizes $\textbf{M}$ and then sort the elements of the resulting vector in ascending order. The vector $\textbf{r}_a$ contains $(1-a)\cdot d_X \cdot T$ zeros followed by $a \cdot d_X \cdot T$ ones, where $a \in [0,1]$. In short, this regularization term encourages the mask to highlight a fraction $a$ of the inputs. For this reason, $a$ can also be referred to as the area of the mask. A first advantage of this approach is that the hyperparameter $a$ can be modulated by the user to highlight a desired fraction of the input. For instance, one can start with a small value of $a$ and slide it to higher values in order to see features gradually appearing in order of importance. Another advantage of this regularization term is that it encourages the mask to be binary, which makes the mask more legible as we will detail in the next section.

Finally, we might want to avoid quick variations in the saliency over time. This could either be a prior belief or a preference of the user with respect to the saliency map. If this is relevant, we can enforce the salient regions to be connected in time with the following loss that penalizes jumps of the saliency over time:
\begin{align*}
 \mathcal{L}_c (\textbf{M}) =~\sum_{t=1}^{T-1} \sum_{i = 1}^{d_X} \mid m_{t+1,i} - m_{t,i}\mid .
\end{align*}
For a fixed fraction $a$, the mask optimization problem can therefore be written as
\begin{align*}
\textbf{M}_a^* = \argmin_{\textbf{M} \in [0,1]^{T \times d_X}} \mathcal{L}_e \left( \textbf{M} \right) + \lambda_a \cdot \mathcal{L}_a \left( \textbf{M} \right) + \lambda_c \cdot \mathcal{L}_c \left( \textbf{M} \right).
\end{align*}  
Note that, in this optimization problem, the fraction $a$ is fixed. In some contexts, one might want to find the smallest fraction of input features that allows us to reproduce the black-box prediction with a given precision. Finding this minimal fraction  corresponds to the following optimization problem:
\begin{align*}
a^* = \min \left\{ a \in [0,1] \mid \mathcal{L}_e \left( \textbf{M}_a^* \right) < \varepsilon \right\},
\end{align*}  
where $\varepsilon$ is the threshold that sets the acceptable precision. The resulting mask $\textbf{M}_{a^*}^*$ is then called \emph{extremal mask}. The idea of an extremal mask is extremely interesting by itself: it explains the black-box prediction in terms of a minimal number of salient features. This is precisely the parsimony that we were referring to in the introduction.

Finally, it is worth mentioning that we have here presented a scheme where the mask preserves the features that minimize the error. There exists a variant of this scheme where the mask preserves the features that maximize the error. This other scheme, together with the detailed optimization algorithm used in our implementation, can be found in Appendix~\ref{sec:implementation_details}.  

\subsection{Masks and information theory}
\label{subsec-info_entropy}

Once the mask is obtained, it highlights the features that contain crucial information for the black-box to issue a prediction. This motivates a parallel between masks and information theory.
To make this rigorous, we notice that the mask admits a natural interpretation in terms of information content. As aforementioned, a value close to $1$ for $m_{t,i}$ indicates that the input feature $x_{t,i}$ carries information that is important for the black box $f$ to predict an outcome. It is therefore natural to interpret a mask coefficient as a probability that the associated feature is salient for the black box to issue its prediction. It is tempting to use this analogy to build the counterpart of \emph{Shannon information content} in order to measure the quantity of information contained in subsequences extracted from the time series. However, this analogy requires a closer analysis. We recall that the Shannon information content of an outcome decreases when the probability of this outcome increases \citep{Shannon1948, MacKay2003, Cover2005}. In our framework, we would like to adapt this notion so that the information content increases with the mask coefficients (if $m_{t,i}$ gets closer to one, this indicates that $x_{t,i}$ carries more useful information for the black-box to issue its prediction). To solve this discrepancy, we have to use $1-m_{t,i}$ as the pseudo-probability appearing in our adapted notion of Shannon information content. 

\begin{definition}[Mask information]
The mask information associated to a mask $\textbf{M}$ and a subsequence $(x_{t,i})_{(t,i) \in A}$ of the input $\textbf{X}$ with $A \subseteq [1:T] \times [1:d_X] $ is
\begin{align*}
I_{\textbf{M}}(A) = - \sum_{(t,i) \in A} \ln \left( 1 - m_{t,i}  \right).
\end{align*}
\end{definition}  
\begin{remark}
Conceptually, the information content of a subsequence measures the quantity of useful information it contains for the black-box to issue a prediction. It allows us to associate a saliency score to a group of inputs according to the mask.
\end{remark}
\begin{remark}
Note that, in theory, the mask information diverges when a mask coefficient is set to one. In practice, this can be avoided by imposing $ \textbf{M} \in (0,1)^{T \times d_X}$.
\end{remark}

As in traditional information theory, the information content is not entirely informative on its own. For instance, consider two subsequences indexed by, respectively, $A , B$ with $ |A| = |B| = 10$. We assume that the submask extracted from $\textbf{M}$ with $A$ contains 3 coefficients $m = 0.9$ and 7 coefficients $m = 0$ so that the information content of $A$ is $I_{\textbf{M}}(A) \approx 6.9$. Now we consider that all the coefficient extracted from $\textbf{M}$ with $B$ are equal to $0.5$ so that $I_{\textbf{M}}(B) \approx 6.9$ and hence $I_{\textbf{M}}(A) \approx I_{\textbf{M}}(B)$. In this example, $A$ clearly identifies 3 important features while $B$ gives a mixed score for the $10$ features. Intuitively, it is pretty clear that the information provided by $A$ is sharper. Unfortunately, the mask information by itself does not allow to distinguish these two subsequences.  Hopefully, a natural distinction is given by the counterpart of Shannon entropy. 

\begin{definition}[Mask entropy]
The mask entropy associated to a mask $\textbf{M}$ and a subsequence $(x_{t,i})_{(t,i) \in A}$ of the input $\textbf{X}$ with $A \subseteq [1:T] \times [1:d_X] $ is
\begin{align*}
S_{\textbf{M}}(A) = - \sum_{(t,i) \in A} m_{t,i} \ln m_{t,i}  + \left( 1 - m_{t,i} \right) \ln \left( 1 - m_{t,i}  \right)
\end{align*}
\end{definition}
\begin{remark}
We stress that the mask entropy is \emph{not} the Shannon entropy for the subsequence $(x_{t,i})_{(t,i) \in A}$. Indeed, the pseudo-probabilities that we consider are not the probabilities $p(x_{t,i})$ for each feature $x_{t,i}$ to occur. In particular, there is no reason to expect that the probability of each input decouple from the others so that it would be wrong to sum the individual contribution of each input separately~\footnote{ It is nonetheless possible to build a single mask coefficient for a group of inputs in order to imitate a joint distribution.} like we are doing here.
\end{remark}

Clearly, it is desirable for our mask to provide explanations with low entropy. This stems from the fact that the entropy is maximized when mask coefficients $m_{t,i}$ are close to $0.5$. In this case, given our probabilistic interpretation, the mask coefficient is ambiguous as it does not really indicate whether the feature is salient. This is consistent with our previous example where $S_{\textbf{M}}(A) \approx 0.98$ while $S_{\textbf{M}}(B) \approx 6.93$ so that $S_{\textbf{M}}(A) \ll S_{\textbf{M}}(B)$. Since masks coefficients take various values in practice, masks with higher entropy appear less legible, as illustrated in Figure~\ref{illustration-mask_information}. Therefore, we use the entropy as a measure of the mask's sharpness and legibility in our experiments. In particular, the entropy is minimized for perfectly binary masks $M \in \{ 0,1\}^{T \times d_X}$, which are easy to read and contain no ambiguity. Consequently, the regularization term $\mathcal{L}_a (\textbf{M})$ that we used in Section~\ref{subsec-optimization} has the effect of reducing the entropy. Since our adapted notions of information and entropy rely on individual contributions from each feature, they come with natural properties.

\begin{figure}[ht]
\vskip 0.2in
\begin{center}
\centerline{\includegraphics[width=0.75\columnwidth]{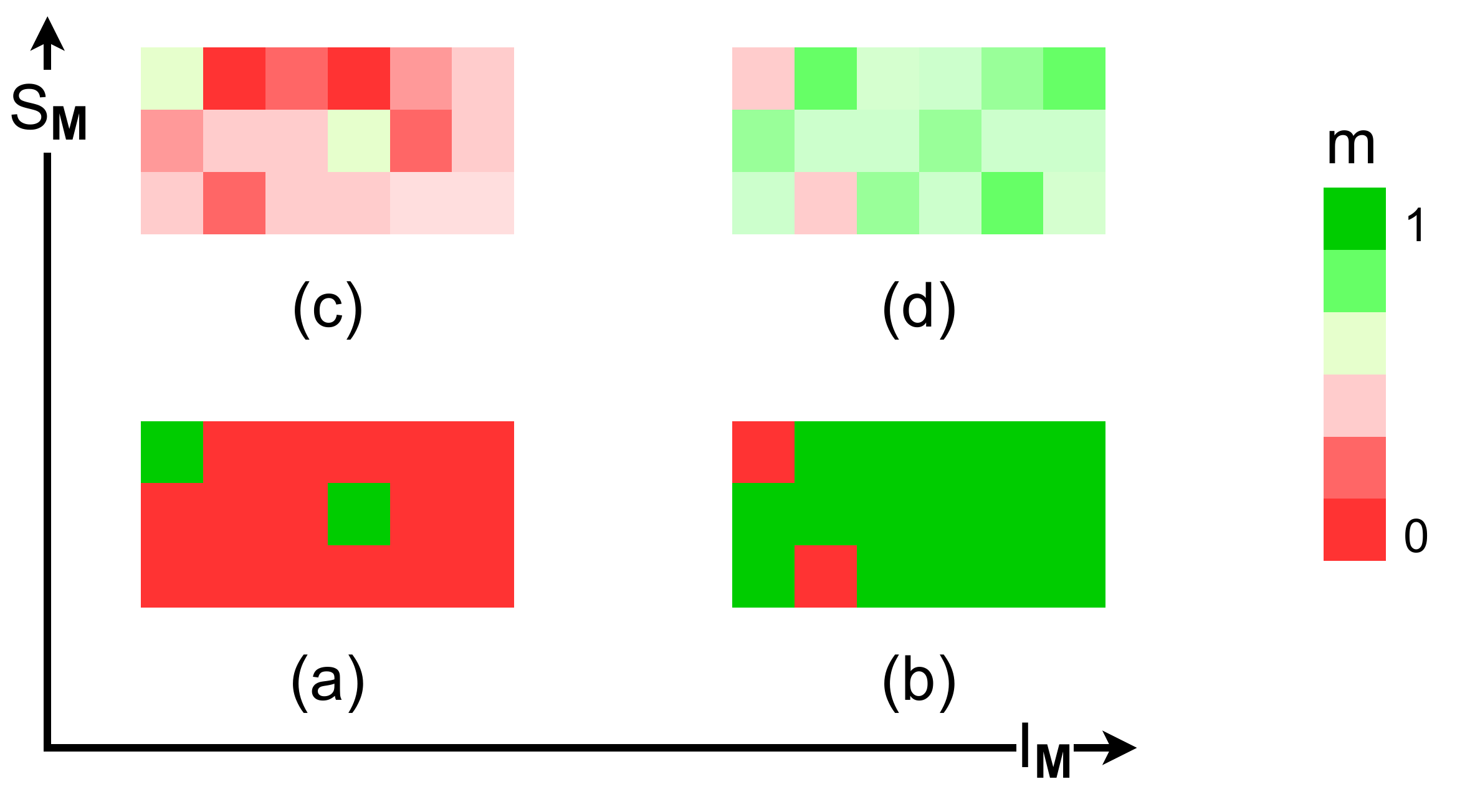}}
\caption{Mask information and entropy. For a given subsequence and a given mask, the information increases when more features are relevant ($a \rightarrow b$ or $c \rightarrow d$). The entropy increases when the sharpness of the saliency map decreases ($a \rightarrow c$ or $b \rightarrow d$). Masks with high entropy appear less legible, especially for long sequences.}
\label{illustration-mask_information}
\end{center}
\vskip -0.2in
\rule[1ex]{\columnwidth}{0.5pt}
\end{figure}

\begin{proposition}[Metric properties]
For all labelling sets $A , B \subseteq [1:T] \times [1:d_X] $, the mask information and entropy enjoy the following properties:

\textbf{1. Positivity}
\begin{align*}
I_{\textbf{M}}(A) \geq 0 \hspace{1cm} S_{\textbf{M}}(A) \geq 0
\end{align*}
\textbf{2. Additivity}
\begin{align*}
& I_{\textbf{M}}(A \cup B) = I_{\textbf{M}}(A) + I_{\textbf{M}}(B) - I_{\textbf{M}}(A \cap B) \\
& S_{\textbf{M}}(A \cup B) = S_{\textbf{M}}(A) + S_{\textbf{M}}(B) - S_{\textbf{M}}(A \cap B)
\end{align*}
\textbf{3. Monotonicity} If $A \subset B$ :
\begin{align*}
I_{\textbf{M}}(A) \leq I_{\textbf{M}}(B) \hspace{1cm} S_{\textbf{M}}(A) \leq S_{\textbf{M}}(B).
\end{align*}
\end{proposition}
\begin{proof}
See Appendix~\ref{sec:math_details}. 
\end{proof}

Together, these properties guarantee that $I_{\textbf{M}}$ and $S_{\textbf{M}}$ define measures for the discrete $\sigma$-algebra $\mathcal{P} \left([1:T]\times[1:d_X] \right)$ on the set of input indexes $[1:T]\times[1:d_X]$. All of these metrics can be computed for any saliency method, hence allowing comparisons in our experiments. For more details, please refer to Appendix~\ref{sec:math_details}.

\section{Experiments}
\label{sec-experiments}
In this section\footnote{Our implementation can be found at \url{https://github.com/JonathanCrabbe/Dynamask}.}, we evaluate the quality of our dynamic masks. There are two big difficulties to keep in mind when evaluating the quality of a saliency method. The first one is that the true importance of features is usually unknown with real-world data. The second one is that the the performance of a saliency method in identifying relevant features depends on the black-box and its performances. In order to illustrate these challenges, we propose three different experiments in ascending order of difficulty. In the first experiment, a white box with known feature importance is used so that both difficulties are avoided. In the second experiment, a black-box trained on a dataset with known feature importance is used and hence only the second difficulty is encountered. In the third experiment, a black-box trained on a real-world clinical dataset is used and hence both difficulties are encountered. For each experiment, more details are given in Appendix~\ref{sec:experiment_details}.

\subsection{Feature importance for a white-box regressor} 
\textbf{Experiment} In this experiment, we work with a trivial white-box regressor whose predictions only rely on a known subset $A = A_T \times A_X \subset [1:T] \times [1:d_X]$ of salient inputs, where $A_T$ and $A_X$ respectively give the salient times and features:
\begin{align*}
\left[f \left( \textbf{X} \right)\right]_{t} =  \left\{
	\begin{array}{cl}
	\sum\limits_{i \in A_X} \left(x_{t,i}\right)^2 & \mbox{if } t \in A_T \\
	0 & \mbox{else.}
	\end{array}
	\right.
\end{align*} 
Note that, in this experiment, $d_Y = 1$ so that we can omit the second index for $f \left( \textbf{X} \right)$. In our experiment, we consider two scenarios that are known to be challenging for state of the art saliency methods~\citep{Ismail2020}. In the first one, the white-box depends on a small portion of salient features $\vert A_X\vert \ll d_X$. In the second one, the white-box depends on a small portion of salient times $\vert A_T\vert \ll T$. We fit a mask to this black-box by using the squared error loss together with the regularization. In our experiment: $T = d_X = 50$. The salient time and features are selected randomly, the input features are generated with an ARMA process. We repeat the experiment 10 times.  

\textbf{Metrics} Since the true salient features are known unambiguously in this setup, we have a large set of metrics at hand to evaluate the performance of a saliency method. To measure the proportion of identified features that are indeed salient, we use the \emph{area under the precision curve} (AUP, higher is better). To measure the portion of salient features that have indeed been identified, we use the \emph{area under the recall curve} (AUR, higher is better). To measure how much information the saliency method predicts for the salient region, we use the mask information content~($I_{\textbf{M}}(A)$, higher is better). To measure the sharpness of the explanation in this region, we use the mask entropy~($S_{\textbf{M}}(A)$, lower is better).

\textbf{Benchmarks} The fact that we are dealing with a white-box regressor already disqualifies some methods such as FIT~\citep{Tonekaboni2020} or DeepLIFT. We compare our method with Feature Occlusion (FO), Feature Permutation (FP), Integrated Gradient (IG) and Shapley Value Sampling (SVS)~\citep{Castro2009}. For fair comparison, we use these baselines to evaluate the importance of each feature at each given time. As explained in Section~\ref{subsec-info_entropy}, a mask can be associated to each saliency method.

\begin{table}[ht]
\caption{Scores for the rare feature experiment.}
\label{table-rare_feature}
\vskip 0.1in
\begin{center}
\resizebox{\columnwidth}{!}{
\begin{tabular}{ c  c c c c }
\toprule
 &  AUP & AUR & $I_{\textbf{M}}(A)$ & $S_{\textbf{M}}(A)$ \\
\hline
\textbf{MASK} & $0.99 \pm 0.01$ & $\textbf{0.58} \pm \textbf{0.03}$ & $\textbf{252} \pm \textbf{69}$ & $\textbf{0.7} \pm \textbf{0.7}$  \\
FO & $\textbf{1.00} \pm \textbf{0.00}$ & $0.14 \pm 0.03$ & $9 \pm 6$ & $11.0 \pm 2.5$  \\
FP & $\textbf{1.00} \pm \textbf{0.00}$ & $0.16 \pm 0.04$ & $13 \pm 7$ & $12.6 \pm 3.3$  \\
IG & $0.99 \pm 0.00$ & $0.14 \pm 0.03$ & $8 \pm 4$ & $11.1 \pm 2.5$  \\
SVS & $\textbf{1.00} \pm \textbf{0.00}$ & $0.14 \pm 0.04$ & $9 \pm 6$ & $11.0 \pm 2.5$  \\
\bottomrule
\end{tabular}}
\end{center}
\vskip -0.1in
\end{table}

\begin{table}[ht]
\caption{Scores for the rare time experiment.}
\label{table-rare_time}
\vskip 0.1in
\begin{center}
\resizebox{\columnwidth}{!}{
\begin{tabular}{ c  c c c c }
\toprule
 & AUP & AUR & $I_{\textbf{M}}(A)$ & $S_{\textbf{M}}(A)$ \\
\hline
\textbf{MASK} & $0.99 \pm 0.01$ & $\textbf{0.68} \pm \textbf{0.04}$ & $\textbf{1290} \pm \textbf{106}$ & $\textbf{7.1} \pm \textbf{2.5}$  \\
FO & $\textbf{1.00} \pm \textbf{0.00}$ & $0.14 \pm 0.04$ & $49 \pm 14$ & $48.3 \pm 6.5$  \\
FP & $\textbf{1.00} \pm \textbf{0.00}$ & $0.16 \pm 0.03$ & $53 \pm 8$ & $54.7 \pm 5.8$  \\
IG & $0.99 \pm 0.00$ & $0.14 \pm 0.04$ & $38 \pm 12$ & $48.7 \pm 6.7$  \\
SVS & $\textbf{1.00} \pm \textbf{0.00}$ & $0.14 \pm 0.04$ & $49 \pm 14$ & $48.3 \pm 6.5$  \\
\bottomrule
\end{tabular}}
\end{center}
\vskip -0.1in
\end{table}

\textbf{Discussion} The AUP is not useful to discriminate between the methods. Our method significantly outperforms all the other benchmarks for all other metric. In particular, we notice that, for both experiments, our method identifies a significantly higher portion of features that are truly salient (higher AUR). As claimed in Section~\ref{subsec-optimization}, we observe that our optimization technique is indeed efficient to produce explanations with low mask entropy. 

\subsection{Feature importance for a black-box classifier }
\textbf{Experiment} We reproduce the state experiment from~\citep{Tonekaboni2020} but in a more challenging setting. In this case, the data is generated according to a 2-state hidden Markov model~(HMM) whose state at time $t \in [1:T]$ (T = 200) is denoted $s_t \in \{ 0,1 \}$. At each time, the input feature vector has three components ($d_X = 3$) and is generated according to the current state via $\textbf{x}_t \sim \mathcal{N}\left( \boldsymbol{\mu}_{s_t}, \boldsymbol{\Sigma}_{s_t} \right)$. To each of these input vectors is associated a binary label $y_t \in \{ 0,1 \}$. This binary label is conditioned by one of the three component of the feature vector, based on the state:
\begin{align*}
p_t = \left\{ \begin{array}{cl}
	\left( 1 + \exp \left[-x_{2,t} \right] \right)^{-1} & \mbox{if } s_t = 0 \\
	\left( 1 + \exp \left[-x_{3,t} \right] \right)^{-1} & \mbox{if } s_t = 1 \\
	\end{array}
	\right. .	
\end{align*}
The label is then emitted via a Bernoulli distribution $y_t~\sim~\mbox{Ber}(p_t)$.
The experiment proposed in~\citep{Tonekaboni2020} focuses on identifying the salient feature at each time where a state transition occurs. However, it is clear from the above discussion that there is exactly one salient feature at any given time. More precisely, the set of salient indexes is $A~=~\left\{~(t, 2+s_t)~\mid~t~\in~[1:T]~ \right\}$. We generate 1000 such time series, 800 of them are used to train a  RNN black-box classifier $f$ with one hidden layer  made of 200 hidden GRU cells\footnote{Implementation details of the original experiment can be found at \url{https://github.com/sanatonek/time_series_explainability}}. We then fit an extremal mask to the black-box by minimizing the cross entropy error for each test time series. We repeat the experiment 5 times. 

\textbf{Metrics} We use the same metrics as before.

\textbf{Benchmarks} In addition to the previous benchmarks, we use Augmented Feature Occlusion (AFO), FIT, RETAIN (RT), DeepLIFT (DL), LIME and GradSHAP (GS).

\begin{table}[ht]
\caption{Scores for the state experiment.}
\label{table-state}
\vskip 0.1in
\begin{center}
\resizebox{\columnwidth}{!}{
\begin{tabular}{ c c c c c }
\toprule
 & AUP & AUR & $I_{\textbf{M}}(A) \div 10^5$ & $S_{\textbf{M}}(A) \div 10^4$ \\
\hline
\textbf{MASK} & $\textbf{0.88} \pm \textbf{0.01}$ & $0.70 \pm 0.00$ & $\textbf{2.24} \pm \textbf{0.01}$ & $\textbf{0.04} \pm \textbf{0.00}$  \\
FO & $0.63 \pm 0.01$ & $0.45 \pm 0.01$ & $0.21 \pm 0.00$ & $1.79 \pm 0.00$  \\
AFO & $0.63 \pm 0.01$ & $0.42 \pm 0.01$ & $0.19 \pm 0.00$ & $1.76 \pm 0.00$  \\
IG & $0.56 \pm 0.00$ & $\textbf{0.78} \pm \textbf{0.00}$ & $0.05 \pm 0.00$ & $1.39 \pm 0.00$  \\
GS & $0.49 \pm 0.00$ & $0.62 \pm 0.00$ & $0.33 \pm 0.00$ & $1.73 \pm 0.00$  \\
LIME & $0.49 \pm 0.01$ & $0.50 \pm 0.01$ & $0.04 \pm 0.00$ & $1.11 \pm 0.00$  \\
DL & $0.57 \pm 0.01$ & $0.20 \pm 0.00$ & $0.09 \pm 0.00$ & $1.18 \pm 0.00$  \\
RT & $0.42 \pm 0.03$ & $0.51 \pm 0.01$ & $0.03 \pm 0.00$ & $1.75 \pm 0.00$  \\
FIT & $0.44 \pm 0.01$ & $0.60 \pm 0.02$ & $0.47 \pm 0.02$ & $1.57 \pm 0.00$  \\
\bottomrule
\end{tabular}}
\end{center}
\vskip -0.1in
\end{table}

\textbf{Discussion} Our method outperforms the other benchmarks for 3 of the 4 metrics.
The AUR suggests that IG identifies more salient features. However, the significantly lower AUP for IG suggests that it identifies too many features as salient. We found that IG identifies $87 \%$ of the inputs as salient\footnote{In this case, we consider a feature as salient if its normalized importance score (or mask coefficient) is above $0.5$.} versus $32 \%$ for our method. From the above discussion, it is clear that only $1/3$ of the inputs are really salient, hence Dynamask offers more parsimony. We use two additional metrics to offer further comparisons between the methods: the area under the receiver operating characteristic (AUROC) and the area under the precision-recall curve (AUPRC). We reach the conclusion that Dynamask outperforms the other saliency methods with $\text{AUROC} = 0.93 \pm 0.00 $ and $\text{AUPRC} = 0.85 \pm 0.00$. The second best method is again Integrated Gradient with $\text{AUROC} = 0.91 \pm 0.00 $ and $\text{AUPRC} = 0.79 \pm 0.00$. All the other methods have $\text{AUROC} \leq 0.85$ and $\text{AUPRC} \leq 0.70$.

\subsection{Feature importance on clinical data}

\textbf{Experiment} We reproduce the MIMIC mortality experiment from \citep{Tonekaboni2020}. We fit a RNN black-box with 200 hidden GRU cells to predict the mortality of a patient based on 48 hours ($T=48$) of patient features ($d_X = 31$). This is a binary classification problem for which the RNN estimates the probability of each class. For each patient, we fit a mask with area $a$ to identify the most important observations $x_{t,i}$. Since the ground-truth feature saliency is unknown, we must find an alternative way to assess the quality of this selection. If these observations are indeed salient, we expect them to have a big impact on the black-box's prediction if they are replaced. For this reason, we replace the most important observations by the time average value for each associated feature $x_{t,i} \mapsto \tilde{x}_{t,i} = \frac{1}{T} \sum_{t=1}^T x_{t,i}$. We then compare the prediction for the original input $f \left(\textbf{X} \right)$ with the prediction for the input where the most important observations have been replaced $f \left( \tilde{\textbf{X}} \right)$. A big shift in the prediction indicates that these observations are salient. We repeat this experiment 3 times for various values of $a$.
 
\textbf{Dataset} We use the MIMIC-III dataset~\citep{Johnson2016}, that contains the health record of  40, 000 ICU de-identified patients at the Beth Israel Deaconess Medical Center. The selected data and its preprocessing is the same as the one done by \citep{Tonekaboni2020}.

\textbf{Metrics} To estimate the importance of the shift induced by replacing the most important observations, we use the \emph{cross-entropy} between $f \left(\textbf{X} \right)$ and $f \left( \tilde{\textbf{X}} \right)$ (CE, higher is better). To evaluate the number of patient whose prediction has been flipped, we compute the \emph{accuracy} of $f \left( \tilde{\textbf{X}} \right)$ with respect to the initial predictions $f \left(\textbf{X} \right)$ (ACC, lower is better).

\textbf{Benchmarks} We use the same benchmarks as before. 

\begin{figure}[ht]
\vskip 0.2in
\begin{center}
\centerline{\includegraphics[width=\columnwidth]{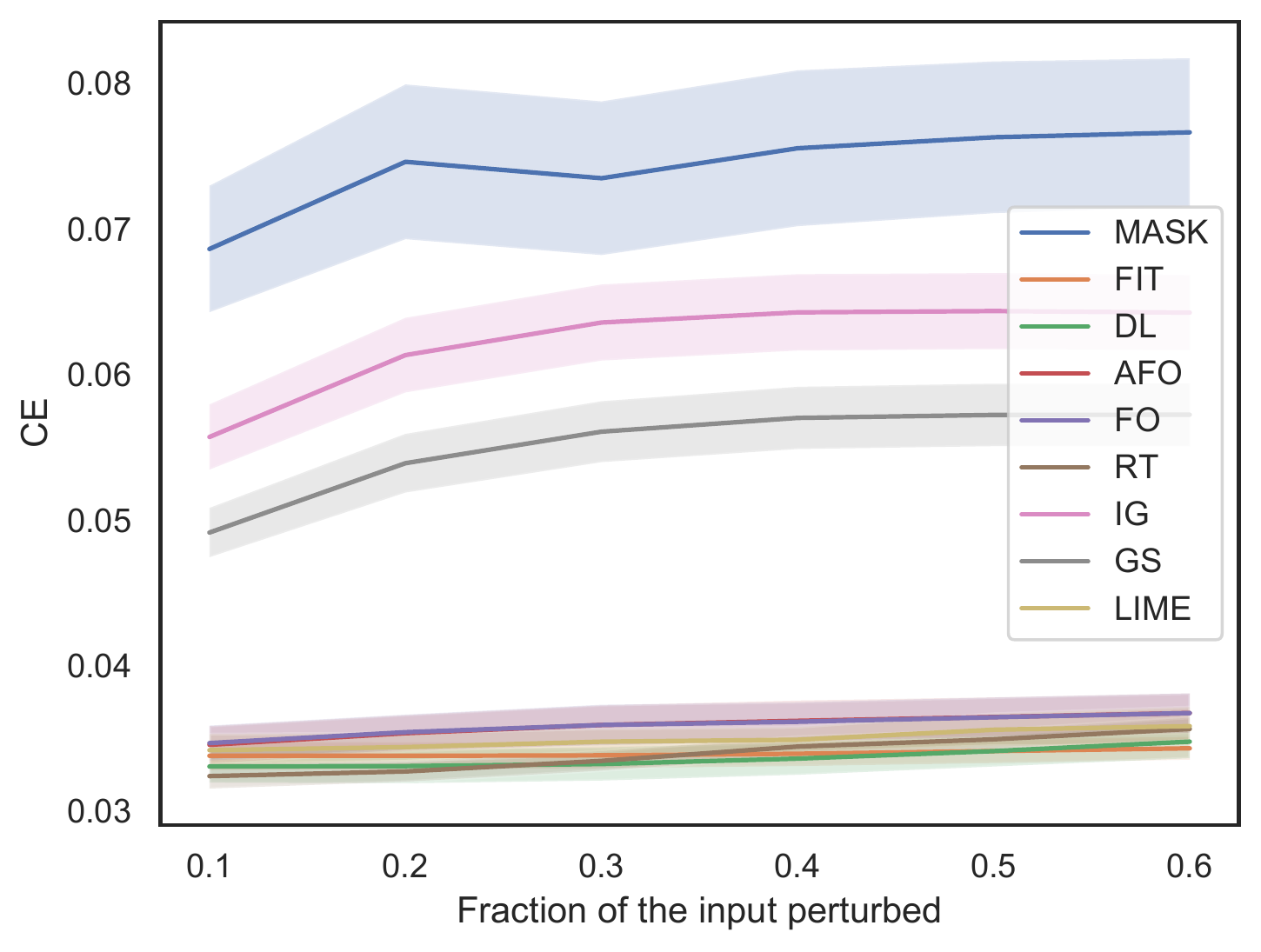}}
\caption{CE as a function of $a$ for the MIMIC experiment.}
\label{fig:mimic_ce}
\end{center}
\vskip -0.2in
\rule[1ex]{\columnwidth}{0.5pt}
\end{figure} 

\begin{figure}[ht]
\vskip 0.2in
\begin{center}
\centerline{\includegraphics[width=\columnwidth]{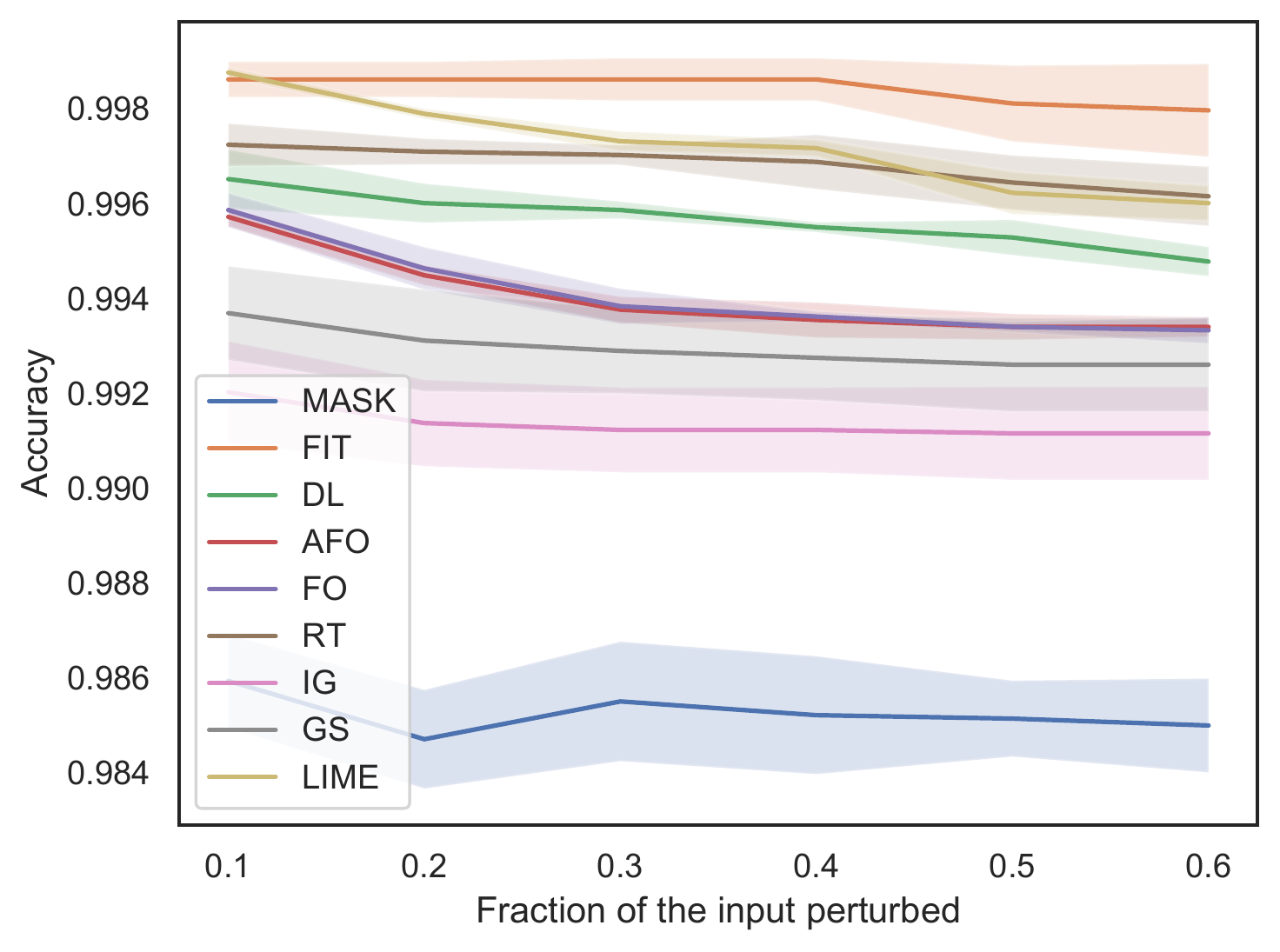}}
\caption{ACC as a function of $a$ for the MIMIC experiment.}
\label{fig:mimic_acc}
\end{center}
\vskip -0.2in
\rule[1ex]{\columnwidth}{0.5pt}
\end{figure}

\textbf{Discussion} The results are shown in Figure~\ref{fig:mimic_ce} \& \ref{fig:mimic_acc}. The observations selected by Dynamask have the most significant impact when replaced. The high accuracy suggests that the perturbation we use (replacing the most important observations by their time average) rarely produces a counterfactual input.  

\section{Conclusion}
\label{sec-conclusion}
In this paper, we introduced Dynamask, a saliency method specifically designed for multivariate time series. These masks are endowed with an insightful information theoretic interpretation and offer a neat improvement in terms of performance. Dynamask has immediate applications in medicine and finance, where black-box predictions require more transparency.

For future works, it would be interesting to design more sophisticated consistency tests for saliency methods in a dynamic setting, like the ones that exist in image classification~\citep{Adebayo2018}. This could be used to study the advantages or disadvantages of our method in more details. Another interesting avenue would be to investigate what the dynamic setting can offer to provide richer explanations with some treatment of causality~\citep{Moraffah2020}.

\section*{Acknowledgments}
The authors are grateful to Ioana Bica, James Jordon, Yao Zhang and the 4 anonymous ICML reviewers for their useful comments on an earlier version of the manuscript. Jonathan Crabbé would like to acknowledge Bilyana Tomova for many insightful discussions and her constant support. Jonathan Crabbé is funded by Aviva. Mihaela van der Schaar is supported by the Office of Naval Research (ONR), NSF 1722516.

\clearpage
\nocite{*}
\bibliography{dynamask}
\bibliographystyle{icml2021}

\clearpage
\appendix
\section{More details on the mathematical formulation} \label{sec:math_details}
\subsection{Proofs}
In this subsection section, we prove the propositions form the main paper.
\begin{proposition}[Metric properties]
For all labelling sets $A , B \subset [1:T] \times [1:d_X] $, the mask information and entropy enjoy the following properties:

\textbf{Positivity}:
\begin{align*}
I_{\textbf{M}}(A) \geq 0 \hspace{1cm} S_{\textbf{M}}(A) \geq 0
\end{align*}
\textbf{Additivity}: 
\begin{align*}
&I_{\textbf{M}}(A \cup B) = I_{\textbf{M}}(A) + I_{\textbf{M}}(B) - I_{\textbf{M}}(A \cap B) \\
&S_{\textbf{M}}(A \cup B) = S_{\textbf{M}}(A) + S_{\textbf{M}}(B) - S_{\textbf{M}}(A \cap B)
\end{align*}
\textbf{Monotonicity} If $A \subset B$ :
\begin{align*}
I_{\textbf{M}}(A) \leq I_{\textbf{M}}(B) \hspace{1cm} S_{\textbf{M}}(A) \leq S_{\textbf{M}}(B).
\end{align*}
\end{proposition}
\begin{proof} Let us proof all the properties one by one.

\textbf{Positivity} By definition, all coefficients from the mask are normalized: $m_{t,i}\in [0,1]$ for $(t,i) \in [1:T] \times [1:d_X]$. Positivity follows trivially from the properties of the logarithm function
\begin{align*}
I_{\textbf{M}}(A) & = - \sum_{(t,i)\in A}  \underbrace{\ln \left(1 - m_{t,i} \right)}_{\leq 0} \\
 & \geq 0 .
\end{align*}
The same goes for the entropy
\begin{align*}
S_{\textbf{M}}(A) & = - \sum_{(t,i)\in A} \underbrace{m_{t,i}}_{\geq 0} \underbrace{\ln m_{t,i}}_{\leq 0} + \underbrace{\left( 1 - m_{t,i} \right)}_{\geq 0} \underbrace{\ln \left(1 - m_{t,i} \right)}_{\leq 0} \\
 & \geq 0 .
\end{align*}

\textbf{Additivity} We first note that the proposition follows trivially if the sets are distinct $A\cap B = \varnothing$:
\begin{align*}
I_{\textbf{M}} \left( A \cup B \right) & =  &- \sum_{(t,i)\in A \cup B} \ln \left(1 - m_{t,i} \right)& \\
& = &- \sum_{(t,i)\in A} \ln \left(1 - m_{t,i} \right) & \\
&   &- \sum_{(t,i)\in B} \ln \left(1 - m_{t,i} \right) & \\
& = &I_{\textbf{M}}(A) + I_{\textbf{M}}(B). & \hspace{1cm}
\end{align*}
Now consider the case where $C \subset D$, since $C$ and $D\setminus C$ are disjoint, we can write
\begin{align}
& I_{\textbf{M}}\left( D \right) = I_{\textbf{M}}\left( C \right) + I_{\textbf{M}}\left( D \setminus C \right) \\
\Rightarrow \ & I_{\textbf{M}}\left( D \setminus C \right) = I_{\textbf{M}}\left( D \right) - I_{\textbf{M}}\left( C \right) . \label{equ-info_subset}
\end{align}

We shall now prove the additivity property in general by using these two ingredients. First we note that the set $A \cup B$ can be written as the disjoint union $A \sqcup \left[ B\setminus(A \cap B) \right]$. It follows that
\begin{align*}
I_{\textbf{M}}(A \cup B) & =  I_{\textbf{M}}\left( A \right) + I_{\textbf{M}}\left( B\setminus \left[A \cap B\right] \right) \\
& =  I_{\textbf{M}}\left( A \right) +  I_{\textbf{M}}\left( B \right) -  I_{\textbf{M}}\left( A \cap B \right),
\end{align*}
where we have used the additivity property for disjoint sets in the first equality and the fact that $\left( A \cap B \right) \subset B$ in the second equality. The same reasoning holds for the entropy.

\textbf{Monotonicity} To prove the monotonicity property, it is useful to note that if $A \subset B$, we can use \eqref{equ-info_subset} to write
\begin{align*}
I_{\textbf{M}}\left( A \right) &= I_{\textbf{M}}\left( B \right) - \underbrace{I_{\textbf{M}}\left( B \setminus A \right)}_{\geq 0} \\
& \leq I_{\textbf{M}}\left( B \right),
\end{align*}
where we have used the information positivity to produce the inequality. The same reasoning holds for the entropy.
\end{proof}

\subsection{Normalized information and entropy}
In this subsection, we introduce the normalized counterparts of our information theoretic metrics. It is important to keep in mind that all the available information for issuing a black-box prediction is in $ [1:T] \times [1:d_X] $. Therefore, the monotonicity property allows to introduce normalized counterparts of the information and the entropy.
\begin{definition}[Normalized metrics]
The normalized mask information associated to a mask $\textbf{M}$ and a subsequence $(x_{t,i})_{(t,i) \in A}$ of the input $\textbf{X}$ with $A \subseteq [1:T] \times [1:d_X] $ is
\begin{align*}
i_{\textbf{M}}(A) = \frac{I_{\textbf{M}}(A)}{I_{\textbf{M}}([1:T] \times [1:d_X])}.
\end{align*}
The same goes for the related normalized mask entropy
\begin{align*}
s_{\textbf{M}}(A) = \frac{S_{\textbf{M}}(A)}{S_{\textbf{M}}([1:T] \times [1:d_X])}.
\end{align*}
\end{definition}
\begin{remark}
By the monotonicity and the positivity properties, it is clear that $ 0 \leq i_{\textbf{M}}(A) , s_{\textbf{M}}(A) \leq 1$ for all $ A \subseteq [1:T] \times [1:d_X] $. This gives a natural interpretation of these quantities as being, respectively, the fraction of the total information and entropy contained in the subsequence $A$ according to the mask $\textbf{M}$.
\end{remark}

The normalized version of the metrics allow to measure what percentage of the total mask information/entropy is contained in a given subsequence.

\subsection{Definition of a mask for other saliency methods}
In this section, we explain how to associate a mask to any saliency method. Suppose that a given method produces a score matrix $\textbf{R} \in \mathbb{R}^{T \times d_{X}}$ that assigns an importance score $r_{t,i}$ for each element $x_{t,i}$ of the input matrix $\textbf{X}$. Then, if we normalize the coefficients of the score matrix, we obtain an associated mask:
\begin{align*}
& \textbf{M} = \frac{1}{r_{max}} \left[ \textbf{R} - r_{min}   \cdot (1)^{T \times d_X} \right], \\
& r_{min} = \min \left\{ r_{t,i} \mid (t,i) \in [1:T] \times [1:d_X]  \right\} \\
& r_{max} = \max \left\{ r_{t,i} \mid (t,i) \in [1:T] \times [1:d_X] \right\} 
\end{align*}
where $(1)^{T \times d_X}$ denotes a $T \times d_X$ matrix with all elements set to $1$. This mask can subsequently be used to compute the mask information content and entropy. In our experiments, we use this correspondence to compare our method with popular saliency methods.
\section{More details on the implementation} \label{sec:implementation_details}
\subsection{Algorithm}
\begin{algorithm}[ht]
   \caption{Dynamask}
   \label{alg-dynamask}
\begin{algorithmic}
   \STATE {\bfseries Input:} input sequence $ \textbf{X} \in \mathbb{R}^{T \times d_X} $, 		    black-box $f$, perturbation operator $\Pi$, mask area $a \in [0,1]$, learning rate 			$\eta \in \mathbb{R}^+$, momentum $\alpha \in \mathbb{R}^+$, initial size regulator $\lambda_0 \in \mathbb{R}^+$, regulator dilation $\delta \in \mathbb{R}_{\geq 1} $, time variation regulator $\lambda_c \in \mathbb{R}^+$ , number of epochs $N \in \mathbb{N}$ 
   \STATE {\bfseries Output:} mask $\textbf{M} \in [0,1]^{T \times d_X} $
   \STATE $ \textbf{M} \leftarrow (0.5)^{T \times d_X} $
   \STATE $ \textbf{r}_a \leftarrow (0)^{T \cdot d_X \cdot (1-a) } \oplus (1)^{T \cdot d_X \cdot a} $
   \STATE $\Delta \textbf{M} \leftarrow 0$
   \STATE $\lambda_a \leftarrow \lambda_0$
   \FOR{$i=1$ {\bfseries to} $N$}
   \STATE $ \tilde{\textbf{X}} \leftarrow \Pi_{\textbf{M}}\left( \textbf{X} \right) $
   \STATE Evaluate the error $\mathcal{L}_e \left( \textbf{M} \right)$ between 					$f(\textbf{X})$ and $f(\tilde{\textbf{X}})$
   \STATE $\mathcal{L}_a \left( \textbf{M} \right) \leftarrow \Vert \text{vecsort}(\textbf{M}) - \textbf{r}_{a} \Vert^2 $ 
   \STATE $\mathcal{L}_c \left( \textbf{M} \right) \leftarrow \sum_{i=1}^{d_X} \sum_{t=1}^{T-1} \vert m_{t+1,i} - m_{t,i} \vert $ 
   \STATE $\Delta \textbf{M} \leftarrow \eta \cdot \nabla_{\textbf{M}} \left[ 
   \mathcal{L}_e  + \lambda_a \mathcal{L}_a + \lambda_c \mathcal{L}_c \right] + \alpha \cdot \Delta \textbf{M} $
   \STATE $\textbf{M} \leftarrow \textbf{M} + \Delta \textbf{M}$
   \STATE $\textbf{M} \leftarrow \text{clamp}_{[0,1]} \left( \textbf{M} \right)$
   \STATE $\lambda_a \leftarrow \lambda_a \times \exp \left( \log \delta / N \right)$
   \ENDFOR
\end{algorithmic}
\end{algorithm}

The mask optimization algorithm is presented in Algorithm~\ref{alg-dynamask}. In the algorithm, we used the notation $(0.5)^{T \times d_X}$ for a $T \times d_X$ matrix with all elements set\footnote{By setting all the initial coefficients of the mask $\textbf{M}$ to $0.5$, we make no prior assumption on the saliency of each feature.} to $0.5$. Similarly, $(0)^{T \cdot d_X \cdot (1-a) }$ denotes a vector with $T \cdot d_X \cdot (1-a)$ components set to $0$ and $ (1)^{T \cdot d_X \cdot a} $ denotes a vector with $T \cdot d_X \cdot a $ components set to $1$. The symbol $\oplus$ denotes the direct sum between two vector spaces, which is equivalent to the concatenation in Algorithm~\ref{alg-dynamask}. The error part of the loss $\mathcal{L}_e$ depends on the task (regression or classification), as explained in Section~3 of the paper. The momentum and the learning rate are typically set to $1$, the number of epoch is typically $1000$.  We also use the clamp function, which is defined component by component as
\begin{align*}
\left[ \text{clamp}_{[0,1]}\left( \textbf{M} \right) \right]_{t,i} = \min \left[ \max \left( m_{t,i} , 0 \right)  , 1 \right].
\end{align*}
Finally, we note that the mask size regularization coefficient $\lambda_a$ grows exponentially during the optimization to reach a maximum value of $\delta \cdot \lambda_0$ at the end of the optimization. In practice, it is initialized to a small value (typically $\lambda_0 = 0.1$) and dilated by several order of magnitude during the optimization (typically $\delta = 1000$). In this way, the optimization procedure works in two times. At the beginning, the loss is dominated by the error $\mathcal{L}_e (\textbf{M})$ so that the mask increases the mask coefficients of salient features. As the regulation coefficient $\lambda_a$ increases, the regulation term becomes more and more important so that the mask coefficients are attracted to $0$ and $1$. At the end of the optimization, the mask is almost binary. 
\subsection{Deletion variant}
We notice that Algorithm~\ref{alg-dynamask} produces a mask that highlights the features that allow to reproduce the black-box prediction by keeping the error part of the loss $\mathcal{L}_e (\textbf{M})$ to be small. However, it is possible to highlight important features in another way. For instance, we could try to find the features that maximizes the prediction shift when perturbed. In this alternative formulation, the mask is obtained by solving the following optimization problem:
\begin{align*}
\tilde{\textbf{M}}_a^* = \argmin_{\textbf{M} \in [0,1]^{T \times d_X}} - \mathcal{L}_e \left(1 - \textbf{M} \right) + \lambda \cdot \mathcal{L}_a \left( \textbf{M} \right).
\end{align*}
Note that, in this case, the sign of the error part is flipped in order to maximizes the shift in the prediction. Moreover, the error is now evaluated for $1 - \textbf{M}$ rather than $\textbf{M}$. This is because important features are maximally perturbed in this case. In this way, a salient feature $x_{t,i}$ can keep a mask coefficient $m_{t,i}$ close to $1$ while being maximally perturbed. The regulator stays the same in this deletion variant, as the mask area still corresponds to the number of mask coefficients set to $1$. We use the deletion variant to obtain the masks in the experiment with clinical data in the main paper.
\section{More details on the experiments} \label{sec:experiment_details}
\subsection{Metrics}
We give the precise definition of each metric that appears in the experiments. Let us start with the metrics that are defined when the true importance is known.

\begin{definition}[AUP,AUR]
Let $\textbf{Q} = (q_{t,i})_{(t,i) \in [1:T] \times [1:d_X]}$ be a matrix in $\{0,1\}^{T \times d_X}$ whose elements indicate the true saliency of the inputs contained in $\textbf{X} \in \mathbb{R}^{T \times d_X}$. By definition, $q_{t,i} = 1$ if the feature $x_{t,i}$ is salient and $0$ otherwise. Let $\textbf{M} = (m_{t,i})_{(t,i) \in [1:T] \times [1:d_X]}$ be a mask in $[0,1]^{T \times d_X}$ obtained with a saliency method. Let $\tau \in (0,1)$ be the detection threshold for $m_{t,i}$ to indicate that the feature $x_{t,i}$ is salient. This allows to convert the mask into an estimator $\hat{\textbf{Q}}(\tau)=\left(\hat{q}_{t,i}(\tau)\right)_{(t,i) \in [1:T] \times [1:d_X]}$ for $\textbf{Q}$ via
\begin{align*}
\hat{q}_{t,i}(\tau) = \left\{
	\begin{array}{ll}
	1 & \text{ if } m_{t,i} \geq \tau \\
	0 & \text{ else.}
	\end{array}
\right.
\end{align*}
Consider the sets of truly salient indexes and the set of indexes selected by the saliency method
\begin{align*}
 A = & \left\{ (t,i) \in [1:T] \times [1:d_X] \mid q_{t,i} = 1 \right\} \\
 \hat{A} ( \tau ) = & \left\{ (t,i) \in [1:T] \times [1:d_X] \mid \hat{q}_{t,i} (\tau) = 1 \right\} .
\end{align*}
We define the precision and recall curves that map each threshold to a precision and recall score:
\begin{align*}
& \text{P} : (0,1)  \longrightarrow  [0,1] : \tau  \longmapsto  \frac{\vert A \cap \hat{A}(\tau) \vert}{\vert \hat{A}(\tau) \vert} \\
& \text{R} : (0,1)  \longrightarrow  [0,1] : \tau  \longmapsto  \frac{\vert A \cap \hat{A}(\tau) \vert}{\vert A \vert}.          
\end{align*} 
The AUP and AUR scores are the area under these curves
\begin{align*}
& \text{AUP} = \int_{0}^1 \text{P}(\tau) d\tau \\
& \text{AUR} = \int_{0}^1 \text{R}(\tau) d\tau.
\end{align*}
\end{definition}
\begin{remark}
Roughly speaking, we consider the identification of salient features as a binary classification task. Each saliency method can thus be seen as a binary classifier for which we compute the AUP and the AUR.
\end{remark}
\begin{remark}
Integrating over several detection thresholds allows to evaluate a saliency method with several levels of tolerance on what is considered as a salient feature.
\end{remark}

In our experiment with MIMIC-III, since the ground true feature importance is unknown, we use the following metrics defined for a binary classification problem\footnote{In our experiment, each input corresponds to a patient. Class $0$ indicates that the patient survives and class $1$ indicates that the patient dies.}.

\begin{definition}[CE, ACC] \label{definition-ce_acc}
Consider a classifier $f$ that maps the input $\textbf{X}$ to a probability $f(\textbf{X}) \in [0,1]$. Let $\tilde{\textbf{X}}$ be a perturbed input produced by a saliency method\footnote{In our experiment, we replace the most important features by the time average of the corresponding feature.}. We define the function that converts a probability into a class
\begin{align*}
\text{class}(p) = \left\{ \begin{array}{ll}	
	0 & \text{if } p < 0.5 \\
	1 & \text{else.}					
\end{array}
\right.
\end{align*}
To measure the shift in the classifier's prediction caused by the perturbation of the input for several test examples $\left\{ \textbf{X}_k \mid k \in [1:K] \right\}$, we use the binary cross-entropy (or log-loss)
\begin{align*}
\text{CE}  = - \frac{1}{K} \sum_{k=1}^{K} & \text{class}\left[ f (\textbf{X}_k) \right] \cdot \log  f\left( \tilde{\textbf{X}}_k \right) \\
 & + \left( 1 -  \text{class}\left[ f (\textbf{X}_k) \right] \right) \cdot \log \left[ 1 -   f\left( \tilde{\textbf{X}}_k \right) \right].
\end{align*}
\end{definition}
To measure the number of prediction flipped by the perturbation, we use the accuracy
\begin{align*}
\text{ACC} = \frac{\vert \left\{ k \in [1:K] : \text{class}\left[ f (\textbf{X}_k) \right] = \text{class}\left[ f (\tilde{\textbf{X}}_k) \right] \right\} \vert}{K}.
\end{align*}
We reproduce our experiment several times to get an average and a standard deviation for all of these metrics.

\subsection{Computing infrastructure}
All our experiments have been performed on a machine with Intel(R) Core(TM) i5-8600K CPU @ 3.60GHz [6 cores] and Nvidia GeForce RTX 2080 Ti GPU.

\subsection{Details on the rare experiment}

\textbf{Data generation} Since this experiment relies on a white-box, we only have to generate the input sequences. As we explain in the main paper, each feature sequence is generated with an ARMA process:
\begin{align*}
x_{t,i} = \varphi_1 \cdot x_{t-1, i} + \varphi_2 \cdot x_{t-2, i} + \varphi_3 \cdot x_{t-3, i} + \epsilon_t,
\end{align*}
with $\varphi_1=0.25$, $\varphi_2=0.1$, $\varphi_3=0.05$ and $\epsilon_t \sim \mathcal{N}(0,1)$. We generate one such sequence with $t \in [1:50]$ for each feature $i \in [1:50]$ by using the Python statsmodels library.

In the rare feature experiment, 5 features are selected as salient. Their indices are contained in $A_X$ and drawn uniformly without replacement from $[1:50]$. The salient times are defined as $A_T = [13:38]$. 

In the rare time experiment, 5 time steps are selected as salient. The initial salient time is drawn uniformly $t^*~\sim~\text{U}([1:46])$. The salient times are then defined as $A_T = [t^*:t^*+4]$. The salient features are defined as $A_X = [13:38]$.

\textbf{Mask fitting} For each time series, we fit a mask by using the temporal Gaussian blur $\pi^g$ as a perturbation operator with $\sigma_{max}=1$ and by using the squared error loss. A mask is fitted for each value of $a \in \left\{ (n+1) \cdot 10^{-3} \mid n \in [0:49] \right\}$. The mask $\textbf{M}^*_a$ with the lowest squared error $\mathcal{L}_e (\textbf{M}^*_a)$ is selected. The hyperparameters for this optimization procedure are $\eta = 1, \alpha = 1, \lambda_0 = 1, \delta = 1000, \lambda_c = 0, N = 1000$.

In our experiments, we don't consider $a > 0.05$. This is because we found experimentally that the error  $\mathcal{L}_e(\textbf{M}_{a}^*)$ generally reaches a plateau as $a$ gets closer to $0.05$, as illustrated in the examples from  Figures~\ref{illustration-rare_feature_area}~\&~\ref{illustration-rare_time_area}. This is consistent with the fraction of inputs that are truly salient since

\begin{align*}
\frac{\vert A \vert}{\vert [1:50]\times[1:50] \vert} = \frac{25 \cdot 5}{50 \cdot 50} = 0.05.
\end{align*}  

\begin{figure}[ht]
\vskip 0.2in
\begin{center}
\centerline{\includegraphics[width=\columnwidth]{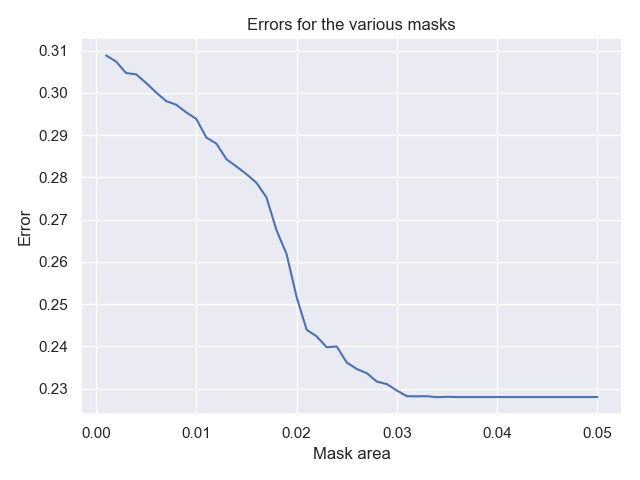}}
\caption{The error $\mathcal{L}_e(\textbf{M}_{a}^*)$ as a function of $a$.  We clearly see that the error stops decreasing when $a$ gets close to $0.05$. This group of masks are fitted on a time series from the rare feature experiment.}
\label{illustration-rare_feature_area}
\end{center}
\vskip -0.2in
\rule[1ex]{\columnwidth}{0.5pt}
\end{figure}

\begin{figure}[ht]
\vskip 0.2in
\begin{center}
\centerline{\includegraphics[width=\columnwidth]{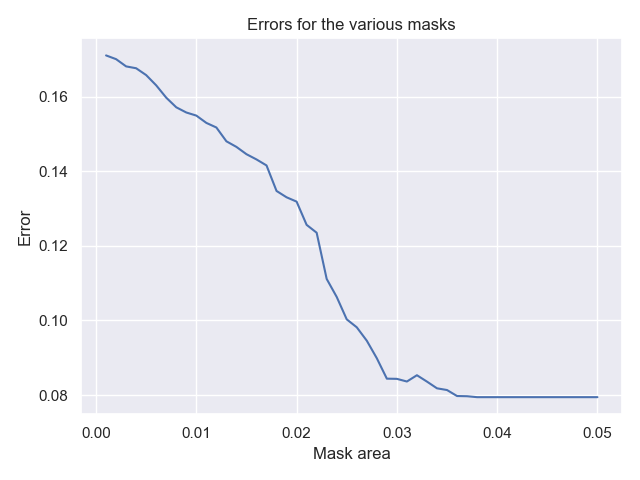}}
\caption{The error $\mathcal{L}_e(\textbf{M}_{a}^*)$ as a function of $a$. We clearly see that the error stops decreasing when $a$ gets close to $0.05$. This group of masks are fitted on a time series from the rare time experiment.}
\label{illustration-rare_time_area}
\end{center}
\vskip -0.2in
\rule[1ex]{\columnwidth}{0.5pt}
\end{figure}

\textbf{Runtime} For rare time, finding the best mask takes on average $15.7$s. For rare feature, finding the best mask takes on average $20.7$s.

\textbf{Illustrations} To illustrate the results of our experiments, we show the saliency masks produced by various methods for the rare feature experiment in Figure~\ref{illustration-rare_feature_example1}~\&~\ref{illustration-rare_feature_example2} and for the rare time experiment in Figure~\ref{illustration-rare_time_example1}~\&~\ref{illustration-rare_time_example2}. For all of these examples, we notice that Dynamask identifies a bigger portion of the truly salient inputs, which illustrates the bigger AUR reported in the main paper.

\subsection{Details on the state experiment}

\textbf{Data generation} The data generation is governed by a Hidden Markov Model (HMM). The initial distribution vector for this HMM is given by $\boldsymbol{\pi}= \left(0.5, 0.5 \right)$ and its transition matrix is
\begin{align*}
\textbf{C} = \begin{pmatrix}
0.1 & 0.9\\
0.1 & 0.9\\
\end{pmatrix}.
\end{align*}
At each time, the input feature vector has three components ($d_X = 3$) and is generated according to the current state via $\textbf{x}_t \sim \mathcal{N}\left( \boldsymbol{\mu}_{s_t}, \boldsymbol{\Sigma}_{s_t} \right)$ with mean vectors depending on the state: $\boldsymbol{\mu}_1= \left(0.1, 1.6, 0.5 \right)$ or $\boldsymbol{\mu}_2= \left(-0.1, -0.4, -1.5 \right)$. When it comes to the covariance matrices, only the off-diagonal terms differ from one state to another:
\begin{align*}
&\boldsymbol{\Sigma}_1 = \begin{pmatrix}
0.8 & 0 & 0\\
0 & 0.8 & 0.01\\
0 & 0.01 & 0.8\\
\end{pmatrix} \\
&\boldsymbol{\Sigma}_2 = \begin{pmatrix}
0.8 & 0.01 & 0\\
0.01 & 0.8 & 0\\
0 & 0 & 0.8 \\
\end{pmatrix}.
\end{align*} 
To each of these input vectors is associated a binary label $y_t \in \{ 0,1 \}$. This binary label is conditioned by one of the three component of the feature vector, based on the state:
\begin{align*}
p_{t} = \left\{ \begin{array}{cl}
	\left( 1 + \exp \left[-x_{2,t} \right] \right)^{-1} & \mbox{if } s_t = 0 \\
	\left( 1 + \exp \left[-x_{3,t} \right] \right)^{-1} & \mbox{if } s_t = 1 \\
	\end{array}
	\right. .	
\end{align*}
The length of each time series is fixed to 200 (T = 200). We generate 1000 such time series, 800 are used for model training and 200 for testing.

\textbf{Model training} We train a RNN with one layer made of 200 GRU cells trained using the Adam optimizer for 80 epochs ($\text{lr} = 0.001, \beta_1=0.9, \beta_2=0.999$ and no weight decay).

\textbf{Mask fitting} For each test time series, we fit a mask by using the temporal Gaussian blur $\pi^g$ as a perturbation operator with $\sigma_{max}=1$. A mask is optimized for each value of $a \in \left\{ 0.15 + 2n \cdot 10^{-2} \mid n \in [0:10] \right\}$. We keep the extremal mask for a threshold set to $\varepsilon = 0.9 \cdot \mathcal{L}_e\left(\textbf{M} = \left( 1 \right)^{T \times d_X} \right)$. The hyperparameters for this optimization procedure are $\eta = 1, \alpha = 1, \lambda_0 = 0.1, \delta = 100, \lambda_c = 1, N = 1000$.

\textbf{Runtime} Finding the extremal mask for a given input takes $49.8$s on average.

\textbf{Illustrations} To illustrate the results of our experiments, we show the saliency masks produced by various methods on Figure~\ref{illustration-example111}~\&~\ref{illustration-example19}. By inspecting these figures, we notice that only Dynamask and Integrated Gradients seem to produce saliency maps where the imprint of the true saliency can be distinguished. One advantage of Dynamask is the contrast put between these salient inputs and the rest. In the case of Integrated Gradients, we see that many irrelevant inputs are assigned an important saliency score, although smaller than the truly salient inputs. This is because gradients are computed individually, without the goal of achieving parsimonious feature selection. \\ 
In addition, we have reported the mask entropy for each of the methods. As claimed in the main paper, Dynamask produces mask that have significantly lower entropy. Among the methods that produce masks with high entropy, we notice two trends. Some methods, such as RETAIN, produce masks where a significant portion of the inputs are assigned a mask entropy close the $0.5$. As discussed in the main paper, these significance of the saliency scores is limited in this situation, since no clear contrast can be drawn between the saliency of different inputs. On the other hand, some methods like FIT produce masks with many different masks coefficients, which renders the saliency map somewhat fuzzy. In both cases, the high entropy detects these obstructions for legibility. 

\subsection{Details on the mimic experiment}
\textbf{Data preprocessing} The data preprocessing used here is precisely the same as the one described in~\citep{Tonekaboni2020}, we summarize it here for completeness. We use the adult ICU admission data from the MIMIC-III dataset~\citep{Johnson2016}. For each patient, we use the features Age, Gender, Ethnicity, First Admission to the ICU, LACTATE, MAGNESIUM, PHOSPHATE, PLATELET, POTASSIUM, PTT, INR, PR, SODIUM, BUN, WBC, HeartRate, DiasBP, SysBP, RespRate, SpO2, Glucose, Temp (in total, $d_X = 31$). The time series data is converted in 48 hour blocks ($T=48$) by averaging all the measurements over each hour block. The patients with all 48 hour blocks missing for a specific features are excluded, this results in 22,9888 ICU admissions. Mean imputation is used when HeartRate, DiasBP, SysBP, RespRate, SpO2, Glucose, Temp are missing. Forward imputation is used when LACTATE, MAGNESIUM, PHOSPHATE, PLATELET, POTASSIUM, PTT, INR, PR, SODIUM, BUN, WBC are missing. All features are standardized and the label is a mortality probability score in [0, 1]. The resulting dataset is split into a training set ($65 \%$), a validation set ($15 \%$) and a test set ($20 \%$).

\textbf{Model training} The model that we train is a RNN with a single layer made of 200 GRU cells. It is trained for 80 epochs with an Adam optimizer ($\text{lr} = 0.001, \beta_1=0.9, \beta_2=0.999$ and no weight decay).

\textbf{Mask fitting} For each test patient, we simply fit a mask with $a = 0.1$ by maximizing the cross-entropy loss in the deletion variant formulation of Dynamask. We use the fade-to-moving average perturbation $\pi^m$ with $W = 48$. The hyperparameters for this optimization procedure are $\eta = 1, \alpha = 1, \lambda_0 = 0.1, \lambda_c = 0, \delta = 1000, N = 1000$.

\textbf{Runtime} Fitting a mask for a given patient takes $3.58$ s on average.

\textbf{Illustrations} To illustrate the results of our experiments, we show the $10\%$ most important features for patients that are predicted to die on Figure~\ref{illustration-patient64}~\&~\ref{illustration-patient144} and for patients that are predicted to survive on Figure~\ref{illustration-patient13}~\&~\ref{illustration-patient666}. In each case, we indicate the cross-entropy between the unperturbed and the perturbed prediction, as defined in Definition~\ref{definition-ce_acc}. We note that Dynamask identifies the features that create the biggest shift. Qualitatively, we note that Dynamask seems to focus much more on the input that appear at latter time. This is consistent with the observations in~\cite{Ismail2019}: these inputs are the most important for the black-box, since it is trained to predict the mortality after the $48$ hours and RNNs have short term memory. 

\subsection{Influence of the perturbation operator}
\begin{table}
\caption{Influence of perturbation operator.}\label{tab-pert}
\vspace{.5cm}
\begin{center}
\resizebox{.5\columnwidth}{!}{
\begin{tabular}{c|ccc}\toprule  
Acc & $\pi^g$ & $\pi^m$ & $\pi^p$  \\ \hline
$\pi^g$ & 1 & .85 &  .80  \\  
$\pi^m$ & .85 & 1 & .80  \\  
$\pi^p$ & .80 & .80 & 1  \\   \midrule
AUROC & .90 & .90 & .86  \\ \bottomrule
\end{tabular}}
\end{center}
\end{table} 
To study the effect of the perturbation operator choice, we have performed the following experiment: in the set-up of the state experiment, we optimize 100 masks on distinct examples by using a Gaussian blur perturbation $(\pi^g, \sigma_{max} = 1)$ and a fade-to-moving average perturbation $(\pi^m, W=3 )$. We do the same with a fade-to-past average perturbation that only uses past values of the features: $(\pi^p, W=6)$. For each pair of perturbation operators, we compute the average accuracy between the associated masks (i.e. the fraction of inputs where both masks agree). For each method, we report the AUROC for the identification of true salient inputs. The results are reported in Table~\ref{tab-pert}. We observe that $\pi^g, \pi^m, \pi^p$ generally agree and offer similar performances.

\begin{figure*}[ht]
\vskip 0.2in
\begin{center}
\centerline{\includegraphics[width=\textwidth]{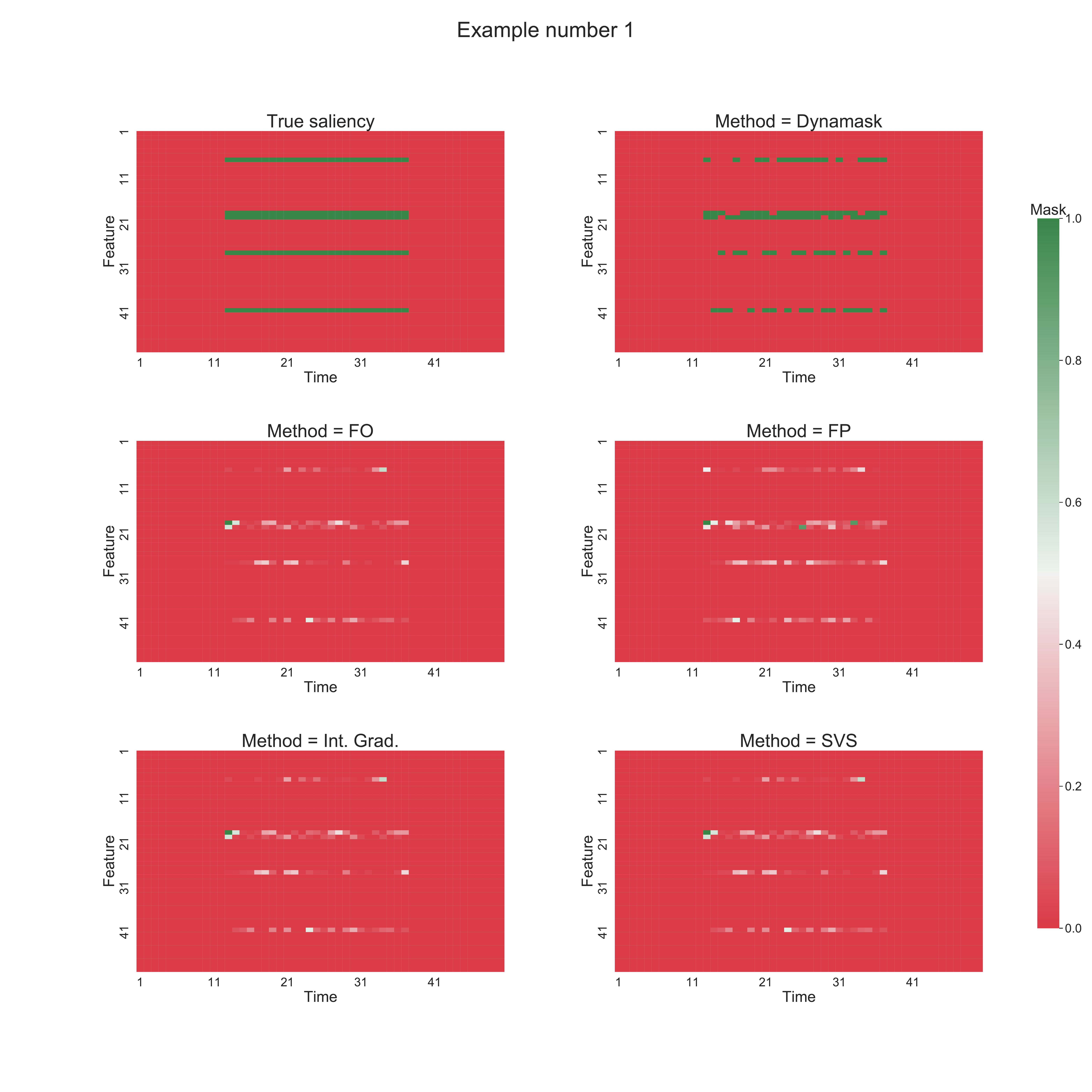}}
\caption{Saliency masks produced by various methods for the test example 1 of the rare feature experiment.}
\label{illustration-rare_feature_example1}
\end{center}
\vskip -0.2in
\rule[1ex]{\textwidth}{0.5pt}
\end{figure*}

\begin{figure*}[ht]
\vskip 0.2in
\begin{center}
\centerline{\includegraphics[width=\textwidth]{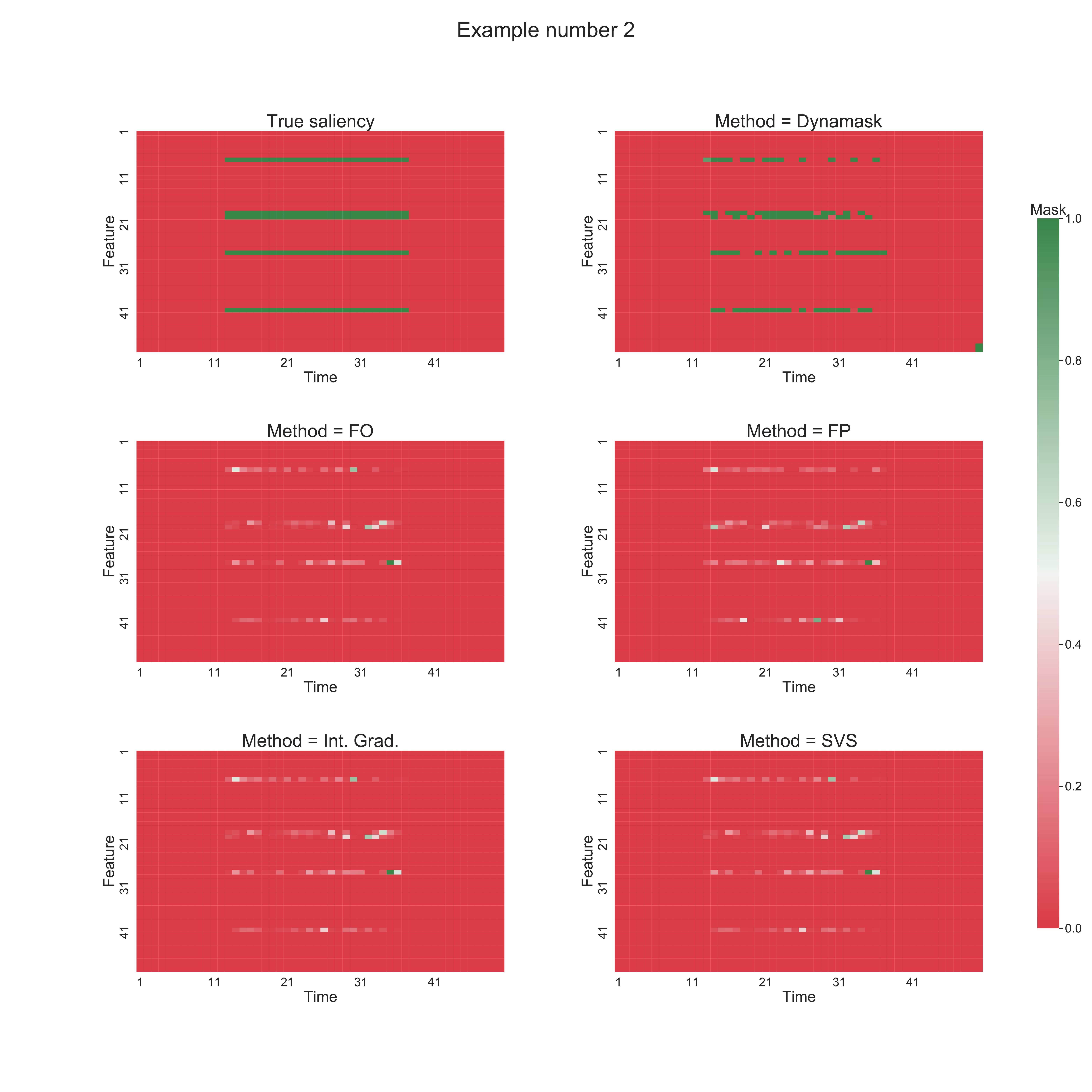}}
\caption{Saliency masks produced by various methods for the test example 2 of the rare feature experiment.}
\label{illustration-rare_feature_example2}
\end{center}
\vskip -0.2in
\rule[1ex]{\textwidth}{0.5pt}
\end{figure*}

\begin{figure*}[ht]
\vskip 0.2in
\begin{center}
\centerline{\includegraphics[width=\textwidth]{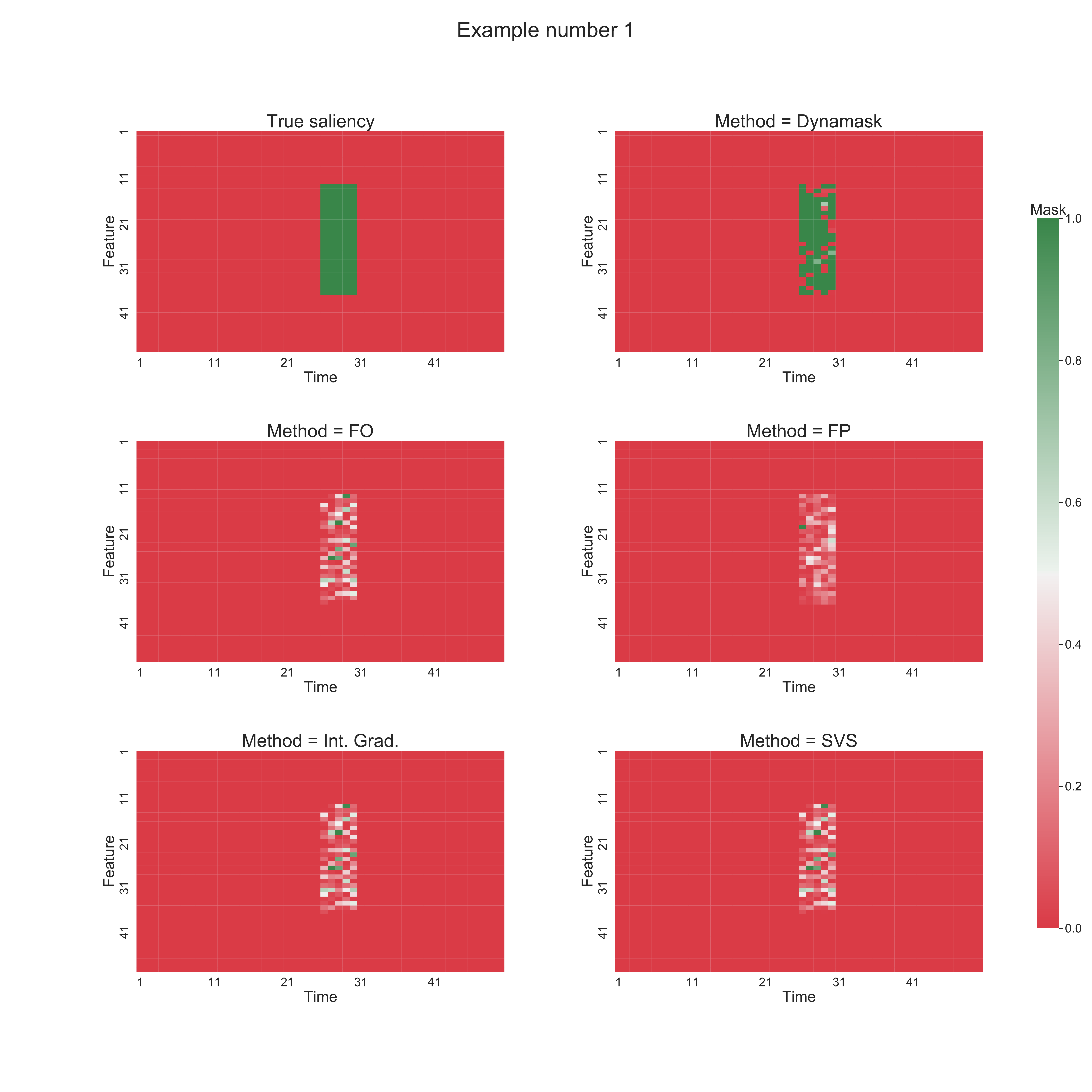}}
\caption{Saliency masks produced by various methods for the test example 1 of the rare time experiment.}
\label{illustration-rare_time_example1}
\end{center}
\vskip -0.2in
\rule[1ex]{\textwidth}{0.5pt}
\end{figure*}

\begin{figure*}[ht]
\vskip 0.2in
\begin{center}
\centerline{\includegraphics[width=\textwidth]{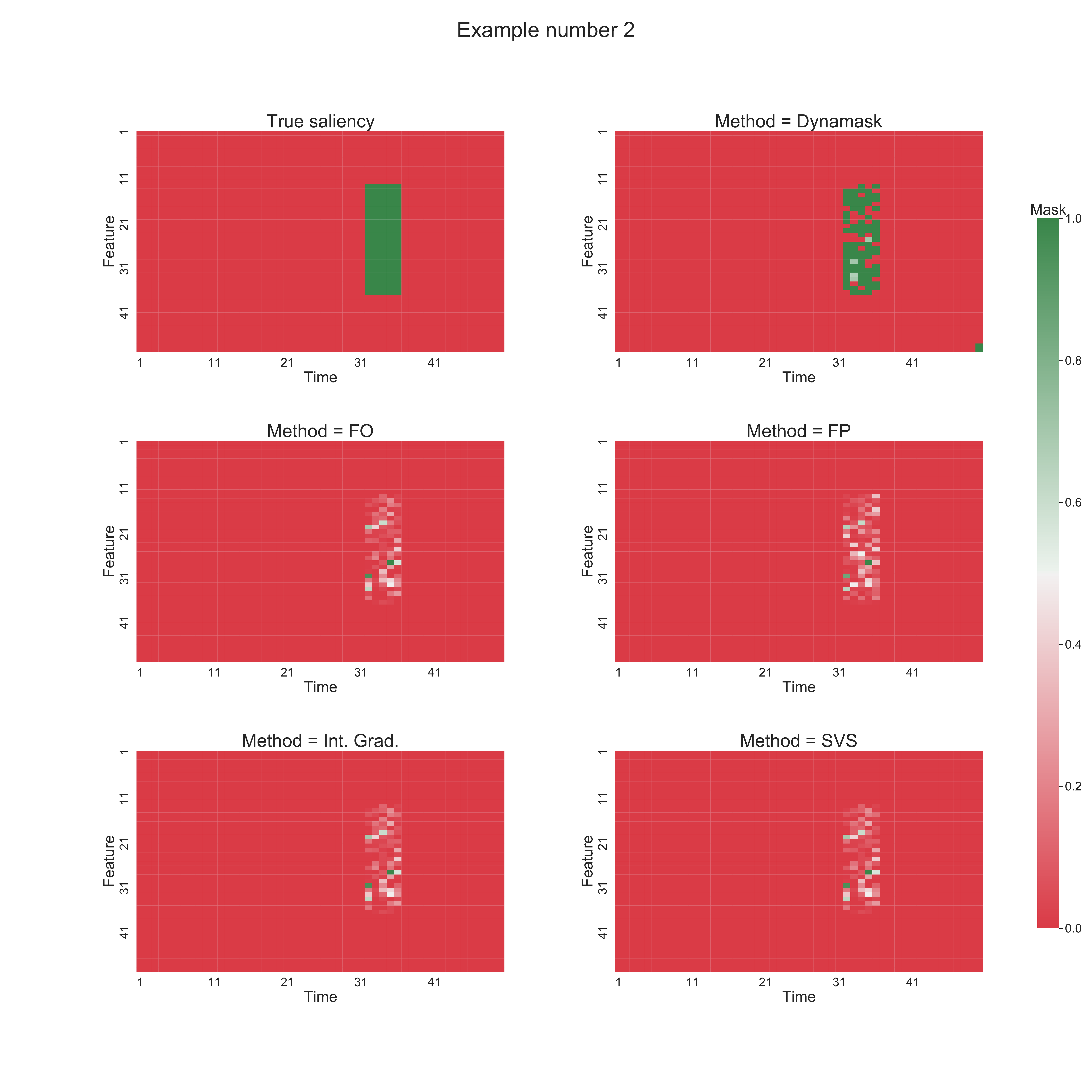}}
\caption{Saliency masks produced by various methods for the test example 2 of the rare time experiment.}
\label{illustration-rare_time_example2}
\end{center}
\vskip -0.2in
\rule[1ex]{\textwidth}{0.5pt}
\end{figure*}

\begin{figure*}[ht]
\vskip 0.2in
\begin{center}
\centerline{\includegraphics[width=\textwidth]{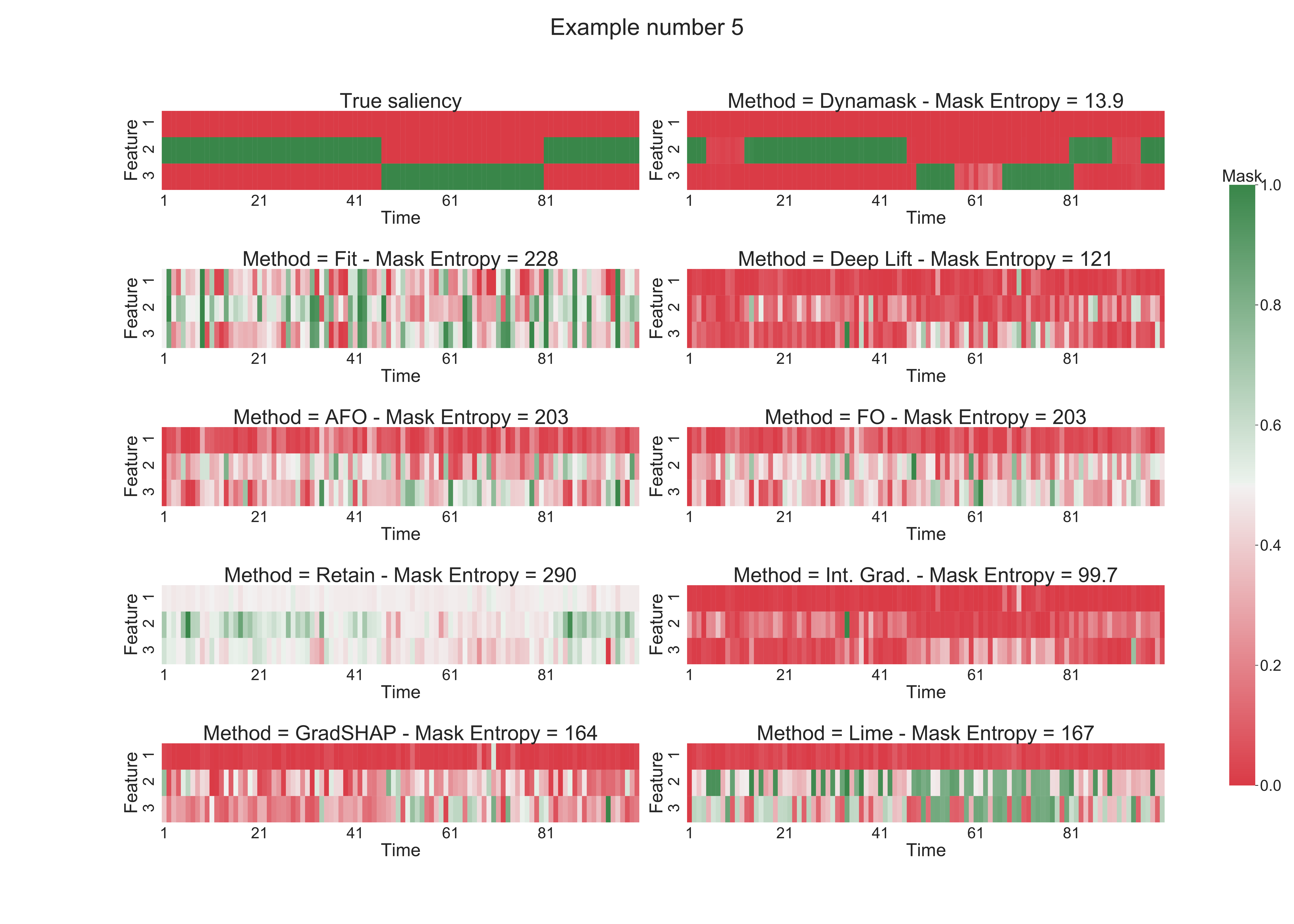}}
\caption{Saliency masks produced by various methods for the test example 5 of the state experiment. For each method, the global entropy of the mask $S_{\textbf{M}} \left( [1:100]\times[1:3] \right)$ is reported.}
\label{illustration-example111}
\end{center}
\vskip -0.2in
\rule[1ex]{\textwidth}{0.5pt}
\end{figure*}

\begin{figure*}[ht]
\vskip 0.2in
\begin{center}
\centerline{\includegraphics[width=\textwidth]{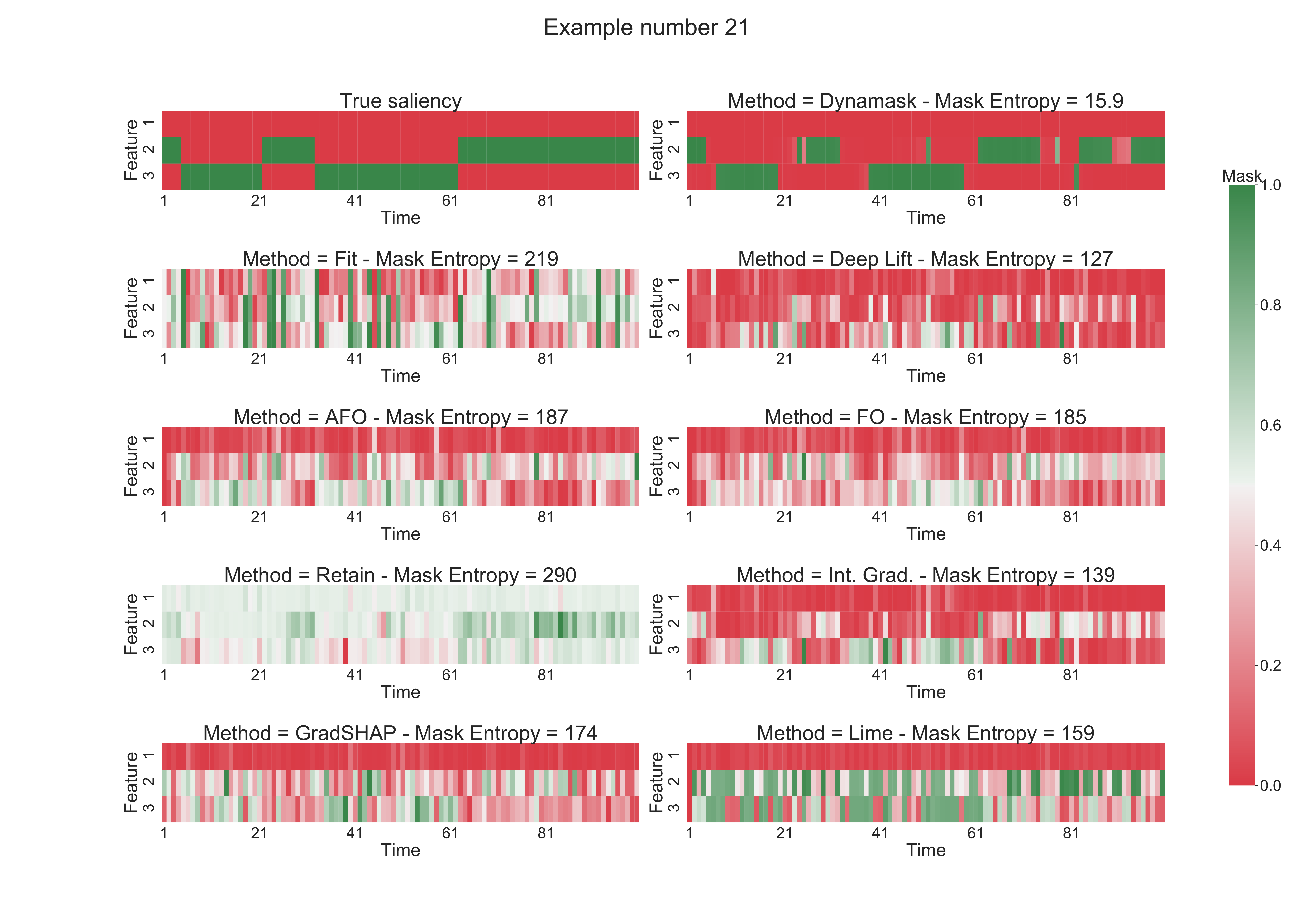}}
\caption{Saliency masks produced by various methods for the test example 21 of the state experiment. For each method, the global entropy of the mask $S_{\textbf{M}} \left( [1:100]\times[1:3] \right)$ is reported.}
\label{illustration-example19}
\end{center}
\vskip -0.2in
\rule[1ex]{\textwidth}{0.5pt}
\end{figure*}

\newpage

\begin{figure*}[ht]
\vskip 0.2in
\begin{center}
\centerline{\includegraphics[width=\textwidth]{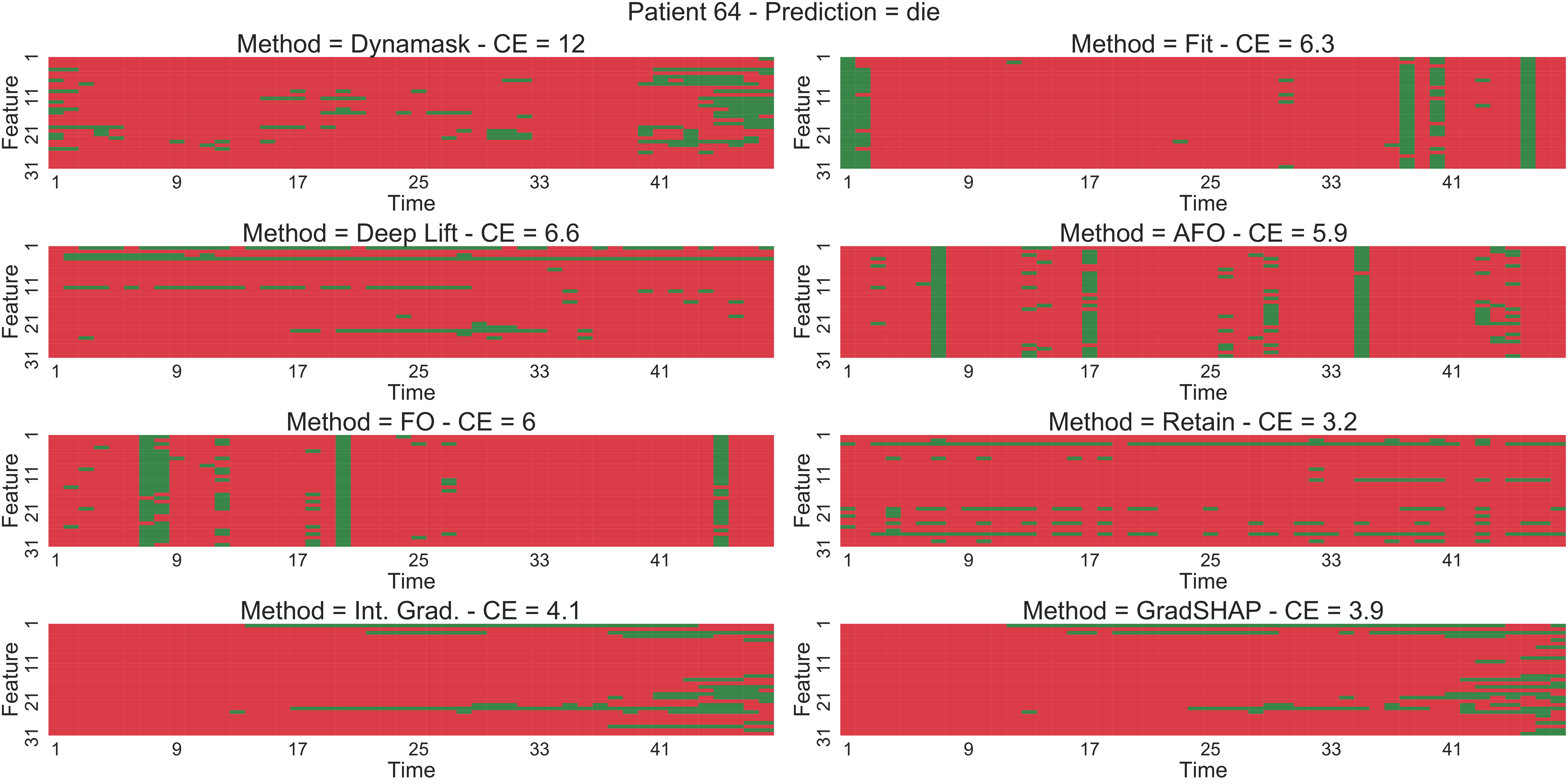}}
\caption{Most important inputs for patient 64. For each saliency method, the $10\%$ most important inputs are represented in green. The black-box predicts that this patient will die. In each case, the cross entropy (CE) between the unperturbed and the perturbed prediction is reported.}
\label{illustration-patient64}
\end{center}
\vskip -0.2in
\rule[1ex]{\textwidth}{0.5pt}
\end{figure*}

\newpage

\begin{figure*}[ht]
\vskip 0.2in
\begin{center}
\centerline{\includegraphics[width=\textwidth]{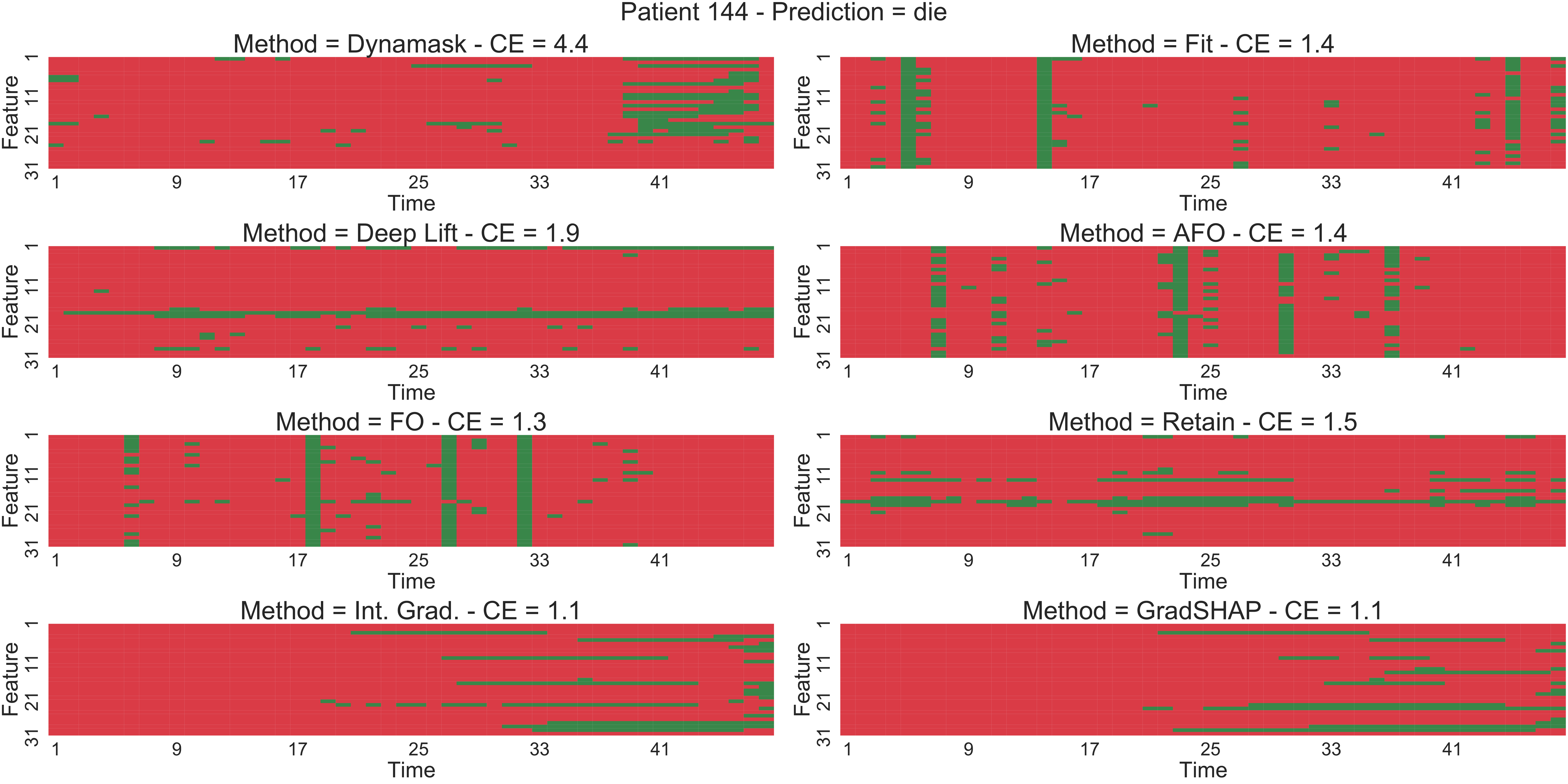}}
\caption{Most important inputs for patient 144. For each saliency method, the $10\%$ most important inputs are represented in green. The black-box predicts that this patient will die. In each case, the cross entropy (CE) between the unperturbed and the perturbed prediction is reported.}
\label{illustration-patient144}
\end{center}
\vskip -0.2in
\rule[1ex]{\textwidth}{0.5pt}
\end{figure*}

\newpage

\begin{figure*}[ht]
\vskip 0.2in
\begin{center}
\centerline{\includegraphics[width=\textwidth]{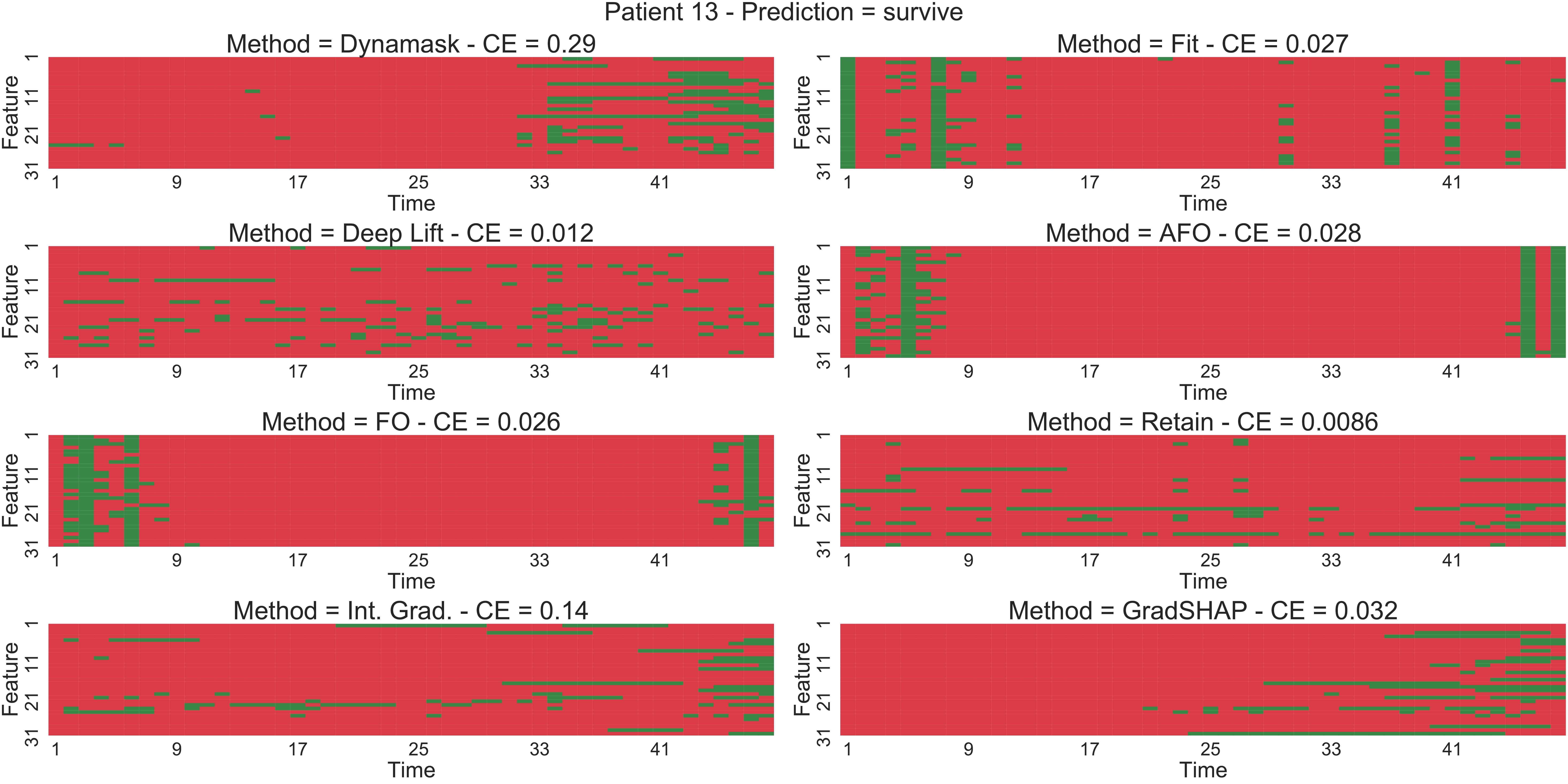}}
\caption{Most important inputs for patient 13. For each saliency method, the $10\%$ most important inputs are represented in green. The black-box predicts that this patient will survive. In each case, the cross entropy (CE) between the unperturbed and the perturbed prediction is reported.}
\label{illustration-patient13}
\end{center}
\vskip -0.2in
\rule[1ex]{\textwidth}{0.5pt}
\end{figure*}

\begin{figure*}[ht]
\vskip 0.2in
\begin{center}
\centerline{\includegraphics[width=\textwidth]{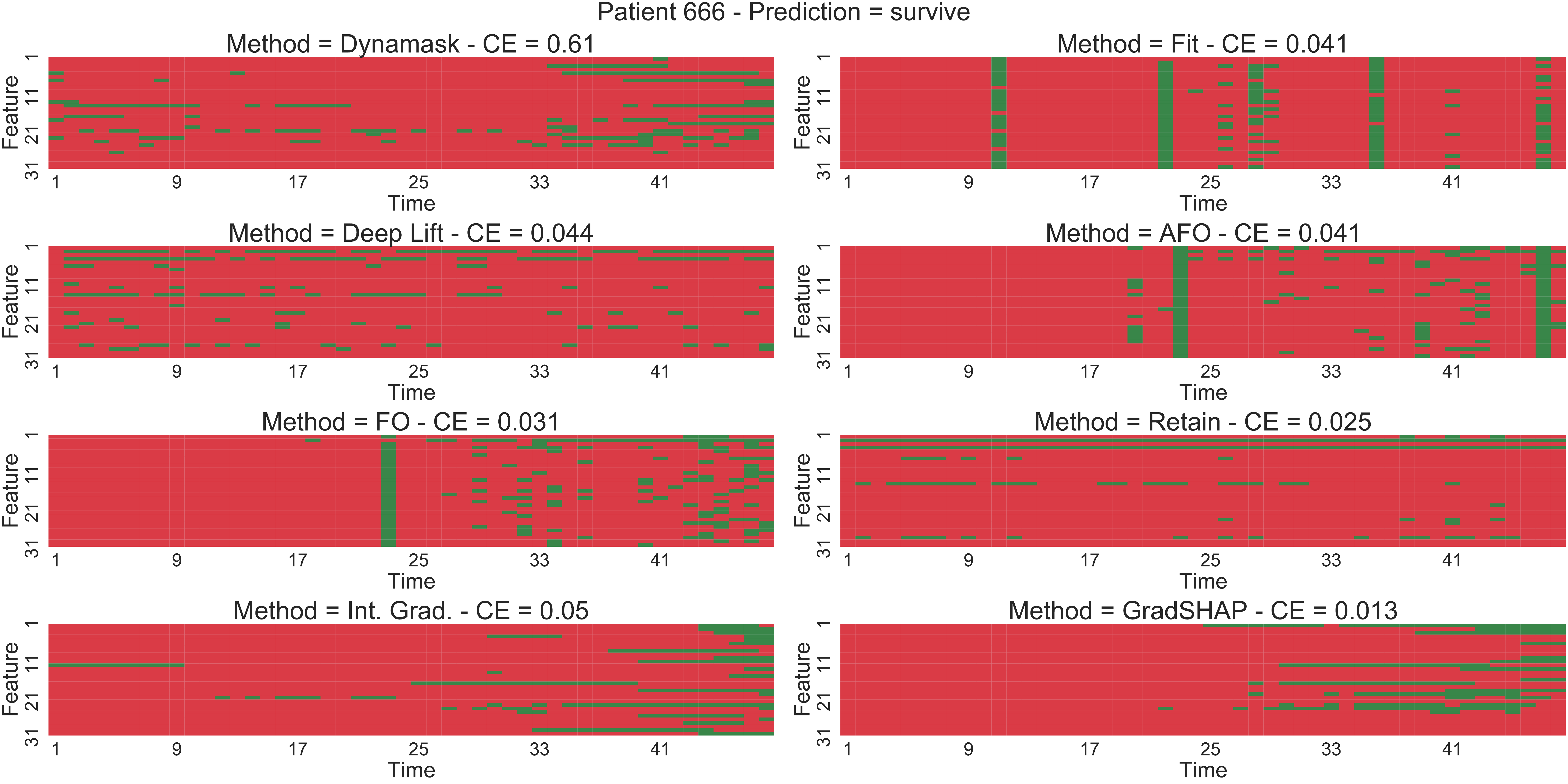}}
\caption{Most important inputs for patient 666. For each saliency method, the $10\%$ most important inputs are represented in green. The black-box predicts that this patient will survive. In each case, the cross entropy (CE) between the unperturbed and the perturbed prediction is reported.}
\label{illustration-patient666}
\end{center}
\vskip -0.2in
\rule[1ex]{\textwidth}{0.5pt}
\end{figure*}

\end{document}